\documentclass{article}

\PassOptionsToPackage{numbers, compress, sort}{natbib}

     \usepackage[final]{neurips_2022}

\usepackage[utf8]{inputenc} 
\usepackage[T1]{fontenc}    
\usepackage{hyperref}       
\usepackage{url}            
\usepackage{booktabs}       
\usepackage{amsfonts}       
\usepackage{nicefrac}       
\usepackage{microtype}      
\usepackage[dvipsnames]{xcolor}         

\usepackage{mathtools}
\usepackage{amsthm}
\usepackage{amsmath,amsfonts}
\usepackage{url}
\usepackage{nicefrac}       
\usepackage{parskip}
\usepackage{multirow}
\usepackage[normalem]{ulem}
 \usepackage{enumitem}
\usepackage{wrapfig}
\usepackage{caption}
\usepackage{chngcntr}

\newtheorem{definition}{Definition}
\newtheorem{proposition}{Proposition}
\newtheorem{lemma}{Lemma}

\newcommand{\softmax}{\mathrm{softmax}}
\newcommand{\multiset}[1]{
  \{\!\!\{#1\}\!\!\}
}
\definecolor{lb}{RGB}{31,119,180}
\definecolor{forestgreen}{HTML}{008000}

\usepackage{textcomp}
\usepackage{phaistos}
\newcommand\myleaf{\mbox{\textleaf}}

\usepackage{xspace}

\newcommand\namecamelcase{Position-encoding injective temporal graph net}
\newcommand\name{position-encoding injective temporal graph net\xspace}
\newcommand\initials{PINT\xspace}

\newcommand\mptgns{MP-TGNs\xspace}
\newcommand\watgns{WA-TGNs\xspace}
\newcommand\mptgn{MP-TGN\xspace}

\newcommand{\tcTCT}{monotone TCT\xspace}

\usepackage{array}
\usepackage{arydshln}
\usepackage{colortbl}%
\newcommand{\myrowcolour}{\rowcolor[gray]{0.925}}

\usepackage{multibib}

\newcites{sup}{Supplementary References}

\usepackage{hyperref}
\hypersetup{
    linkbordercolor=lb,
    urlbordercolor=lb,
    citebordercolor=green,
}

\usepackage{tcolorbox}
\usepackage{tabularx}
\usepackage{colortbl}
\tcbuselibrary{skins}

\tcbset{tab2/.style={enhanced,fonttitle=\bfseries,
colback=white,colframe=gray,colbacktitle=white,
coltitle=black,center title}}

\title{Provably expressive temporal graph networks}

\author{%
  Amauri H. Souza$^1$, \quad Diego Mesquita$^2$, \quad Samuel Kaski$^{1,3}$, \quad Vikas Garg$^{1, 4}$ \\
  $^1$Aalto University \quad $^2$Getulio Vargas Foundation \quad $^3$University of Manchester \quad $^4$YaiYai Ltd \\
  {\small 
  \texttt{\{amauri.souza, samuel.kaski\}@aalto.fi, diego.mesquita@fgv.br, vgarg@csail.mit.edu}
  } 
}

\begin{document}

\maketitle

\begin{abstract}
Temporal graph networks (TGNs) have gained prominence as models for embedding dynamic interactions,  but little is known about their theoretical underpinnings. We establish fundamental results about the representational power and limits of the two main categories of TGNs: those that aggregate temporal walks (WA-TGNs), and those that augment local message passing with recurrent memory modules (MP-TGNs). Specifically, novel constructions reveal the inadequacy of MP-TGNs and WA-TGNs, proving that neither category subsumes the other. We extend the 1-WL (Weisfeiler-Leman) test to temporal graphs, and show that the most powerful MP-TGNs should use injective updates, as in this case they become as expressive as the temporal WL. Also, we show that sufficiently deep MP-TGNs cannot benefit from memory, and MP/WA-TGNs fail to compute graph properties such as girth. 
 
These theoretical insights lead us to PINT --- a novel architecture that leverages injective temporal message passing and relative positional features. Importantly, PINT is provably more expressive than both MP-TGNs and WA-TGNs.
PINT significantly outperforms existing TGNs on several real-world benchmarks.
\end{abstract}

\section{Introduction}
\label{sec:introduction}
Graph neural networks (GNNs) \citep{Gori2005, scarselli2009,survey, Modular2022} have recently led to breakthroughs in many applications \citep{ Pinion2021, ComplexPhysics, antibiotic_design} by resorting to message passing between neighboring nodes in input graphs. 
While message passing imposes an important inductive bias, it does not account for the dynamic nature of interactions in time-evolving graphs arising from many real-world domains such as social networks and bioinformatics \citep{jodie, Xu2019}. 
In several scenarios, these temporal graphs are only given as a sequence of timestamped events.
Recently, temporal graph nets (TGNs)~\citep{jodie, tgn,  DyRep, caw, tgat} have emerged as a prominent learning framework for temporal graphs and have become particularly popular due to their outstanding predictive performance.
Aiming at capturing meaningful structural and temporal patterns, TGNs combine a variety of building blocks, such as self-attention~\citep{Vaswani2017,gat}, time encoders~\citep{Kazemi2019, DaXu2019}, recurrent models~\citep{GRU,lstm}, and message passing  \citep{Gilmer2017}.

Unraveling the learning capabilities of (temporal) graph networks is imperative to understanding their strengths and pitfalls, and designing better, more nuanced models that are both theoretically well-grounded and practically efficacious. For instance, the enhanced expressivity of higher-order GNNs has roots in the inadequacy of standard message-passing GNNs to separate graphs that  are indistinguishable by the Weisfeiler-Leman isomorphism test, known as 1-WL test or color refinement algorithm \citep{Maron2019, Morris2019, Sato2019, Vignac2020, gin}. 
Similarly, many other notable advances on GNNs were made possible by untangling their ability to generalize \citep{Garg2020, liao2021, Verma2019}, extrapolate \citep{xu2021how}, compute graph properties \citep{chen2020, dehmamy2019understanding, Garg2020}, 
and express Boolean classifiers \citep{Barcelo2020}; by uncovering their connections to distributed algorithms \citep{Loukas2020, Sato2019}, graph kernels \citep{du2019}, dynamic programming \citep{Xu2020What}, diffusion processes \citep{Grand21}, graphical models \citep{factorGNNs}, and combinatorial optimization \citep{Cappart21}; and by analyzing their discriminative power \citep{Loukas2020-how-hard, pmlr-v119-nguyen20c}. In stark contrast, the theoretical foundations of TGNs remain largely unexplored. For instance, unresolved questions include: How does the expressive power of existing TGNs compare? When do TGNs fail? Can we improve the expressiveness of TGNs? What are the limits on the power of TGNs?   

\begin{figure}[t]
\begin{minipage}{0.35\textwidth}
    \centering
\usetikzlibrary{shapes,snakes,backgrounds,arrows}
\resizebox{0.85\textwidth}{!}{
\begin{tikzpicture}[scale=0.6]

        \node [draw, ellipse, very thick, rotate=45, orange, fill=orange, fill opacity=1.0, minimum width=7.1cm, minimum height=5.1cm, align=center] at   (3,0)   {};
        
        \node [draw, ellipse, very thick, orange, rotate=135,fill=orange, fill opacity=1.0, minimum width=7.1cm, minimum height=5.1cm, align=center] at   (0,0)   {};

        \node [draw, opacity=0.0, ellipse, rotate=135,fill=white, fill opacity=1.0, minimum width=7cm, minimum height=5cm, align=center] at   (0,0)   {};

        \node [draw, ellipse, opacity=0.0, rotate=45, fill=white, fill opacity=1.0, minimum width=7cm, minimum height=5cm, align=center] at   (3,0)   {};

        \node [align=center, orange] at (1.2,6){\LARGE \bf \initials};

        \node [draw, ellipse, opacity=0.0, rotate=45, lb, fill=lb, fill opacity=0.2, minimum width=7cm, minimum height=5cm, align=center] at   (3,0)   {};
        \node [align=center, lb] at (5.5,3.8){\bf \Large CAW};

        \node [draw, opacity=0.0, ellipse, thick, loosely dashed, rotate=135, Mulberry, fill=Mulberry, fill opacity=0.2, minimum width=7cm, minimum height=5cm, align=center] at   (0,0)   {};
        \node [align=center, Orchid!50!black] at (-2.15,3.5){\bf \Large Injective \\ \bf \Large \mptgns};


        \node [draw, ellipse, thick, opacity=0.0, loosely dashed, fill=forestgreen, fill opacity=0.2, rotate=135, forestgreen, minimum width=3.5cm, minimum height=2.5cm, align=center] at   (0,0)   {};
        \node [align=center, forestgreen, rotate =-45] at (0,0){\bf \Large SOTA \\\bf \Large\mptgns};

        \node [draw, ellipse, thick, opacity=0.0, loosely dashed, fill=darkgray, fill opacity=0.2, rotate=180, darkgray, minimum width=3.5cm, minimum height=2.5cm, align=center] at   (1.2,-8)   {};
        \node [align=center, darkgray] at (1.2,-8){\bf\Large Open \\\bf \Large problems};


    \end{tikzpicture}
}
\end{minipage}
\begin{minipage}{0.65\textwidth}
\centering
\small
\input{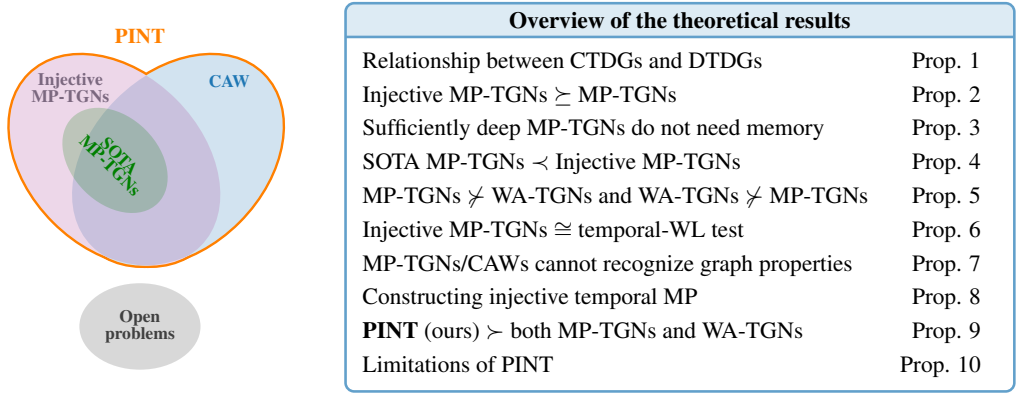}
\end{minipage}
\caption{{Schematic diagram and summary of our contributions}.}
\label{fig:test}
\end{figure}

We establish a series of results to address these fundamental questions. We begin by showing that discrete-time dynamic graphs (DTDGs) can always be converted to continuous-time analogues (CTDGs) without loss of information, so we can focus on analyzing the ability of TGNs to distinguish nodes/links of CTDGs.  We consider a general framework for message-passing TGNs (\mptgns) \citep{tgn} that subsumes a wide variety of methods \citep[e.g.,][]{tgat, jodie, DyRep}.  We prove that equipping \mptgns with injective aggregation and update functions leads to the class of most expressive \textit{anonymous MP-TGNs} (i.e., those that do not leverage node ids). 
Extending the color-refinement algorithm  to temporal settings, we show that these most powerful MP-TGNs are as expressive as the temporal WL method. Notably, existing \mptgns do not enforce injectivity. 
We also delineate the role of memory in \mptgns: nodes in a network with only a few layers of message passing fail to aggregate information from a sufficiently wide receptive field (i.e., from distant nodes), so memory serves to offset this highly local view with additional global information. In contrast, sufficiently deep architectures obviate the need for memory modules.

Different from \mptgns, walk-aggregating TGNs (\watgns) such as CAW \citep{caw} obtain representations from anonymized temporal walks. We provide constructions that expose shortcomings of each framework, establishing that \watgns 
can distinguish links in cases where \mptgns fail and vice-versa. Consequently, neither class is more
expressive than the other. 
Additionally, we show that \mptgns and CAWs cannot decide temporal graph properties such as diameter, girth, or number of cycles. Strikingly, our analysis unravels the subtle relationship between the walk computations in CAWs and the MP steps in \mptgns.

Equipped with these theoretical insights, we propose \initials (short for \emph{\name}), founded on a new temporal layer that leverages the strengths of both \mptgns and \watgns. Like the most expressive MP-TGNs, PINT defines injective message  passing and update steps. PINT also augments memory states with  novel relative positional features, and these features can replicate all the discriminative benefits available to \watgns. Interestingly, the time complexity of computing our positional features is less severe than the sampling overhead in CAW, thus PINT can often be trained faster than CAW. Importantly, we establish that \initials is provably more expressive than CAW as well as \mptgns. 

\textbf{Our contributions} are three-fold:
\begin{itemize}
\item a rigorous theoretical foundation for TGNs is laid - elucidating the role of memory, benefits of injective message passing, limits of existing TGN models, temporal extension of the 1-WL test and its implications, impossibility results about temporal graph properties, and the relationship between main classes of TGNs --- as summarized in \autoref{fig:test};
\item explicit injective temporal functions are introduced, and a novel method for temporal graphs is proposed that is provably more expressive than state-of-the-art TGNs;
\item extensive empirical investigations underscore practical benefits of this work. The proposed method is either competitive or significantly better than existing models on several real benchmarks for dynamic link prediction, in transductive as well as inductive settings.  
\end{itemize}

\section{Preliminaries}

We denote a \emph{static graph} $G$ as a tuple $(V, E, \mathcal{X}, \mathcal{E})$, where $V=\{1, 2, \ldots, n\}$ denotes the set of nodes and $E \subseteq V \times V$ the set of edges. Each node $u \in V$ has a feature vector $x_u \in \mathcal{X}$ and each edge $(u,v) \in E$ has a feature vector $e_{uv} \in \mathcal{E}$, where $\mathcal{X}$ and $\mathcal{E}$ are countable sets of features. 

\paragraph{Dynamic graphs} can be roughly split according to their discrete- or continuous-time nature \cite{Kazemi2020}. 
A \emph{discrete-time dynamic graph} (DTDG) is of a sequence of graph snapshots $(G_1, G_2, \dots)$ usually sampled at regular intervals, each snapshot being a static graph $G_t=(V_t, E_t, \mathcal{X}_t, \mathcal{E}_t)$. 
%

A \emph{continuous-time dynamic graph} (CTDG) evolves with  node- and edge-level \emph{events},
such as addition and deletion.
We represent a CTDG as a sequence of time-stamped multi-graphs $(\mathsf{G}(t_0), \mathsf{G}(t_1), \dots)$ 
such that $t_k<t_{k+1}$, and $\mathsf{G}(t_{k+1})$ results from updating $\mathsf{G}(t_{k})$ with all events at time $t_{k+1}$.
We assume no event occurs between $t_k$ and $t_{k+1}$. 
We denote an interaction (i.e., edge addition event) between nodes $u$ and $v$ at time $t$ as a tuple $(u, v, t)$ associated with a feature vector $e_{u v}(t)$.
Unless otherwise stated, interactions correspond to undirected edges, i.e., $(u, v, t)$ is a shorthand for $(\{u, v\}, t)$.

Noting that CTDGs allow for finer (irregular) temporal resolution, we now formalize the intuition that DTDGs can be reduced to and thus analyzed as CTDGs, but the converse may need extra assumptions. 

\begin{proposition}[Relationship between DTDG and CTDG]\label{prop:equivalence}
For any DTDG we can build a CTDG with the same sets of node and edge features that contains the same information, i.e., we can reconstruct the original DTDG from the converted CTDG.
The converse holds if the CTDG timestamps form a subset of a uniformly spaced countable set.
\end{proposition} 

 Following the usual practice \citep{caw,jodie,tgat}, we focus on CTDGs with edge addition events (see \autoref{ap:deletion} for a discussion on deletion). 
Thus, we can represent temporal graphs as sets 
$\mathcal{G}(t)=\{(u_{k}, v_{k}, t_{k}) ~|~ t_k < t\}$. We also assume each distinct node $v$ in $\mathcal{G}(t)$ has an initial feature vector $x_v$. 

%
%

\paragraph{Message-passing temporal graph nets (\mptgns).}
\citet{tgn} introduced \mptgn as a general representation learning framework for temporal graphs. 
The goal is to encode the graph dynamics into node embeddings, capturing information that is relevant for the task at hand.
To achieve this, \mptgns rely on three main ingredients: memory, aggregation, and update. 
Memory comprises a set of vectors that summarizes the history of each node, and is updated using a recurrent model whenever an event occurs. 
The aggregation and update components resemble those in message-passing GNNs, where the embedding of each node is refined using messages from its neighbors.

We define the \emph{temporal neighbohood} of node $v$ at time $t$ as $\mathcal{N}(v, t) = \{ (u, e_{u v} (t^\prime), t^\prime) ~|~ \exists (u, v, t^\prime) \in \mathcal{G}(t)  \}$, i.e., the set of neighbor/feature/timestamp triplets from all interactions of node $v$ prior to $t$.
\mptgns compute the temporal representation $h_v^{(\ell)}(t)$ of $v$ at layer $\ell$ by recursively applying
\begin{align}
    \tilde{h}_v^{(\ell)}(t) & = \textsc{Agg}^{(\ell)}(\{\!\!\{ (h_u^{(\ell-1)}(t), t-t^\prime, e) \mid (u, e, t^\prime) \in \mathcal{N}(v, t)\}\!\!\}) \label{eq:aggregation}\\
    h_v^{(\ell)}(t) & = \textsc{Update}^{(\ell)}\left(h_v^{(\ell-1)}(t), \tilde{h}_v^{(\ell)}(t)\right) \label{eq:update},
\end{align}
where $\{\!\!\{ \cdot \}\!\!\}$ denotes multisets, $h_v^{(0)}(t) = s_v(t)$ is the {\em state} of $v$ at time $t$, and $\textsc{Agg}^{(\ell)}$ and $\textsc{Update}^{(\ell)}$ are arbitrary parameterized functions. The memory block updates the states as events occur.
Let $\mathcal{J}(v, t)$ be the set of events involving $v$ at time $t$.
%
The state of $v$ is updated due to $\mathcal{J}(v, t)$ as 
\begin{align}
m_v(t) &= \textsc{MemAgg}(\{\!\!\{[s_v(t), s_u(t), t-t_v, e_{vu}(t)] \mid (v, u, t) \in  \mathcal{J}(v, t)\}\!\!\}) \label{eq:mem-msg}
\\
s_v(t^+) &= \textsc{MemUpdate}(s_v(t), m_v(t)), \label{eq:memory}
\end{align}
where $s_v(0)=x_v$ (initial node features), $s_{v}(t^{+})$ denotes the updated state of $v$ due to events at time $t$, and $t_{v}$ denotes the time of the last update to $v$. \textsc{MemAgg} combines information from  simultaneous events involving node $v$ and \textsc{MemUpdate} usually implements a gated recurrent unit (GRU) \citep{GRU}. Notably, some \mptgns do not use memory, or equivalently, they employ \emph{identity memory}, i.e., $s_v(t) = x_v$ for all $t$. 
We refer to \autoref{ap:egn} for further details. 

\paragraph{Causal Anonymous Walks (CAWs).}  \citet{caw} proposed CAW as an approach for link prediction on temporal graphs. To predict if an event $(u, v, t)$ occurs, CAW first obtains sets $S_u$ and $S_v$ of temporal walks starting at nodes $u$ and $v$ at time $t$.
An $(L-1)$-length \emph{temporal walk} is represented as $W=((w_1, t_1), (w_2, t_2), \dots, (w_L, t_L))$, with $t_1 > t_2 > \dots > t_L$ and $(w_{i-1}, w_i, t_i) \in \mathcal{G}(t)$ $\forall i > 1$.
Note that when predicting $(u, v, t)$, we have walks starting at time $t_1=t$.
Then, CAW anonymizes walks replacing each node $w$ with a set $I_\text{CAW}(w; S_u, S_v)= \{g(w; S_u), g(w; S_v)\}$ of two feature vectors. The $\ell$-th entry of $g(w; S_u)$ stores how many times $w$ appears at the $\ell$-th position in a walk of $S_u$, i.e. $g(w, S_u)[\ell]=|\{W \in S_u : (w, t_{\ell}) = W_\ell \}|$ where $W_\ell$ is $\ell$-th pair of $W$.

To encode a walk $W$ with respect to the sets $S_u$ and $S_v$, CAW applies $\textsc{Enc}(W; S_u, S_v)= \mathrm{RNN}([f_1(I_{\text{CAW}}(w_i; S_u, S_v)) \| f_2(t_{i-1}-t_i)]_{i=1}^{L})$ where $f_1$ is a permutation-invariant function, $f_2$ is a time encoder, and $t_0=t_1=t$.   
Finally, CAW combines the embeddings of each walk in $S_u 
\cup S_v$ using mean-pooling or self-attention to obtain the representation for the event $(u,v,t)$. 

In practice, TGNs often rely on sampling schemes for computational reasons. However we are concerned with the expressiveness of TGNs, so our analysis assumes complete structural information, i.e., $S_u$ is the set of all temporal walks from $u$ and \mptgns combine information from all neighbors.

\section{The representational power and limits of TGNs}

We now study the expressiveness of TGNs on node/edge-level prediction. 
We also establish connections to a variant of the WL test and show limits of specific TGN models. Proofs are in \autoref{ap:proofs}.

\subsection{Distinguishing nodes with \mptgns}

We analyze \mptgns w.r.t. their ability to map different nodes to different locations in the embedding space. In particular, we say that an $L$-layer \mptgn distinguishes two nodes $u, v$ of a temporal graph at time $t$, if the last layer embeddings of $u$ and $v$ are different, i.e., $h^{(L)}_{u}(t) \neq h^{(L)}_{v}(t)$.

We can describe the MP computations of a node $v$ at time $t$ via its \textit{temporal computation tree} (TCT) $T_v(t)$. $T_v(t)$ has $v$ as its root and height equal to the number of MP-TGN layers $L$. We will keep the dependence on depth $L$ implicit for notational simplicity.  
For each element $(u, e, t^\prime) \in\mathcal{N}(v, t)$ associated with $v$, we have a node, say $i$, in the next layer of the TCT linked to the root by an edge annotated with $(e, t^\prime)$. 
The remaining TCT layers are built recursively using the same mechanism.
We denote by $\sharp^t_v$ the (possibly many-to-one) operator that maps nodes in $T_v(t)$ back to nodes in $\mathcal{G}(t)$, e.g., $\sharp^t_v i = u$. 
Each node $i$ in $T_v(t)$ has a state vector $s_i = s_{
\sharp^t_v i}(t)$.
To get the embedding of the root $v$, information is propagated bottom-up, i.e., starting from the leaves all the way up to the root --- each node aggregates the message from the layer below and updates its representation along the way. Whenever clear from context, we denote $\sharp^t_v$ simply as $\sharp$ for a cleaner notation.

We study the expressive power of \mptgns through the lens of functions on multisets adapted to temporal settings, i.e., comprising triplets of node states, edge features, and timestamps. Intuitively, injective functions `preserve' the information as it is propagated, so should be essential for maximally expressive \mptgns. We formalize this idea in \autoref{lemma:1} and \autoref{prop:injectiveness-requirement} via  \autoref{def:isomorphism}.

\vspace{-2pt}
\begin{definition}[Isomorphic TCTs] \label{def:isomorphism} 
Two TCTs $T_z(t)$ and $T_{z^\prime}(t)$ at time $t$ are isomorphic if there is a bijection $f: V(T_{z}(t)) \rightarrow V(T_{z^\prime}(t))$ between the nodes of the trees such that the following holds: \vspace{-2pt}
\begin{enumerate}[leftmargin=0.1cm]
    \item[] $(u, v, t^\prime) \in E(T_{z}(t)) \Longleftrightarrow (f(u), f(v), t^\prime) \in E(T_{z^\prime}(t))$ \vspace{-2pt}
    \item[] $\forall (u, v, t^\prime) \in E(T_{z}(t)): e_{uv}(t^\prime) = e_{f(u) f(v)}(t^\prime)$ and $\forall u \in V(T_{z}(t)): s_u = s_{f(u)}$ and $k_{u} = k_{f(u)}$
\end{enumerate}
Here, $k_{u}$ denotes the level (depth) of node $u$ in the tree. The root node has level $0$, and for a node $u$ with level $k_{u}$, the children of $u$ have level $k_{u} + 1$.
\end{definition}

\vspace{-1pt}
\begin{lemma} \label{lemma:1} 
If an \mptgn $Q$ with $L$ layers distinguishes two nodes $u, v$ of a dynamic graph $\mathcal{G}(t)$, then the $L$-depth TCTs $T_u(t)$ and $T_v(t)$ are not isomorphic.
\end{lemma}

For non-isomorphic TCTs, \autoref{prop:injectiveness-requirement} shows that improving \mptgns with \emph{injective} message passing layers suffices to achieve node distinguishability, extending results from static GNNs \citep{gin}. 

\begin{proposition}[Most expressive \mptgns]\label{prop:injectiveness-requirement}
If the L-depth TCTs of two nodes $u, v$ of a temporal graph $\mathcal{G}(t)$ at time $t$ are not isomorphic, then an \mptgn $Q$ with $L$ layers and injective aggregation and update functions at each layer is able to distinguish nodes $u$ and $v$.
\end{proposition}

So far, we have considered TCTs with general memory modules, i.e.,  nodes are annotated with memory states. However, an important question remains: \emph{How does the expressive power of \mptgns change as a function of the memory}? Our next result - \autoref{prop:memory_expressivity} - shows that adding GRU-based memory does not increase the expressiveness of suitably deep \mptgns.

\begin{proposition}[The role of memory] \label{prop:memory_expressivity}
Let $\mathcal{Q}^{[M]}_L$ denote the class of \mptgns with recurrent memory and $L$ layers. Similarly, we denote by $\mathcal{Q}_L$ the family of memoryless \mptgns with $L$ layers. Let $\Delta$ be the temporal diameter of $\mathcal{G}(t)$ (see \autoref{def:diam}). Then, it holds that: 
\vspace{-6pt}
\begin{enumerate}[noitemsep]
    \item If $L < \Delta$: $\mathcal{Q}^{[M]}_L$ is strictly more powerful than $\mathcal{Q}_L$ in distinguishing nodes of $\mathcal{G}(t)$;
    \item For any $L$ : $\mathcal{Q}_{L+\Delta}$ is at least as powerful as $\mathcal{Q}^{[M]}_{L}$ in distinguishing nodes of $\mathcal{G}(t)$.
\end{enumerate}
\end{proposition}

The \mptgn framework is rather general and subsumes many modern methods for temporal graphs \citep[e.g., ][]{tgat,jodie,DyRep}. We now analyze the theoretical limitations of two concrete instances of \mptgns: TGAT \citep{tgat} and TGN-Att \citep{tgn}. Remarkably, these models are among the best-performing \mptgns. Nonetheless, we can show that there are nodes of very simple temporal graphs that TGAT and TGN-Att cannot distinguish (see \autoref{fig:limits-tgns}). We formalize this in \autoref{prop:TGAT-TGN-Att} by establishing that there are cases in which TGNs with injective layers can succeed, but TGAT and TGN-Att cannot.

\begin{proposition}[Limitations of TGAT/TGN-Att]\label{prop:TGAT-TGN-Att}
There exist temporal graphs containing nodes $u, v$ that have non-isomorphic TCTs, yet no TGAT nor TGN-Att with mean memory aggregator (i.e., using \textsc{Mean} as \textsc{MemAgg}) can distinguish $u$ and $v$.
\end{proposition}

\begin{figure}[tb]
    \centering
    \includegraphics[width=\textwidth]{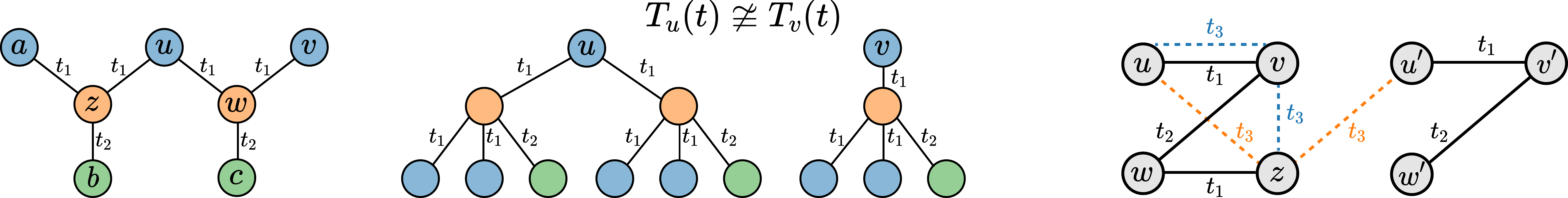}
    \caption{\textbf{Limitations of TGNs}. [\textit{Left}] Temporal graph with nodes $u$, $v$ that TGN-Att/TGAT cannot distinguish. Colors are node features, edge features are identical, and $t_3 > t_2 > t_1$. [\textit{Center}] TCTs of $u$ and $v$  are non-isomorphic. However, the attention layers of TGAT/TGN-Att compute weighted averages over a same multiset of values, returning identical messages for $u$ and $v$. [\textit{Right}] MP-TGNs fail to distinguish the events $\color{lb}(u, v, t_3)\color{black}$ and $\color{lb}(v, z, t_3)\color{lb}$ as TCTs of $z$ and $u$ are isomorphic. Meanwhile, CAW cannot separate $\color{orange}(u, z, t_3)\color{black}$ and $\color{orange}(u^\prime, z, t_3)\color{lb}$: the $3$-depth TCTs of $u$ and $u^\prime$ are not isomorphic, but the temporal walks from $u$ and $u^\prime$ have length $1$, keeping CAW from capturing structural differences.}
    \label{fig:limits-tgns}
\vspace{-.35cm}
\end{figure}

This limitation stems from the fact that the attention mechanism employed by TGAT and TGN-Att is proportion invariant \citep{Perez2019}. The memory module of TGN-Att cannot counteract this limitation due to its mean-based aggregation scheme. We provide more details in Appendix \ref{ap:limitations_TGAT}.  

\subsection{Predicting temporal links}
Models for dynamic graphs are usually trained and evaluated on temporal link prediction \citep{Kleinberg2007}, which consists in predicting whether an event would occur at a given time. To predict an event between nodes $u$ and $v$ at $t$, \mptgns combine the node embeddings $h^{(L)}_u(t)$ and $h^{(L)}_v(t)$, and push the resulting vector through an MLP. On the other hand, CAW is originally designed for link prediction tasks and directly computes edge embeddings, bypasssing the computation of node representations.

We can extend the notion of node distinguishability to edges/events. We say that a model distinguishes two synchronous events $\gamma=(u, v, t)$ and $\gamma^\prime = (u^\prime, v^\prime, t)$ of a temporal graph if it assigns different edge embeddings $h_{\gamma} \neq h_{\gamma^\prime}$ for $\gamma$ and $\gamma^\prime$. \autoref{prop:tgns_vs_caw} asserts that CAWs are not strictly more expressive than \mptgns, and vice-versa. Intuitively, CAW's advantage over \mptgns lies in its ability to exploit node identities and capture correlation between walks. However, CAW imposes temporal constraints on random walks, i.e., walks have timestamps in decreasing order, which can limit its ability to distinguish events. \autoref{fig:limits-tgns}(Right) sketches constructions for \autoref{prop:tgns_vs_caw}. 

\begin{proposition}[Limitations of \mptgns and CAW]\label{prop:tgns_vs_caw}
There exist distinct synchronous events of a temporal graph that CAW can distinguish but \mptgns with injective layers cannot, and vice-versa.
\end{proposition}

\subsection{Connections with the WL test}
\label{subsec:temporal_wl}
The Weisfeiler-Leman test (1-WL) has been used as a key tool to analyze the expressive power of GNNs. We now study the power of \mptgns under a temporally-extended version of 1-WL, and prove negative results regarding whether TGNs can recognize properties of temporal graphs. 

\paragraph{Temporal WL test.} We can extend the WL test for temporal settings in a straightforward manner by exploiting the equivalence between temporal graphs and multi-graphs with timestamped edges \citep{Orsini2015}. In particular, the temporal variant of 1-WL assigns colors for all nodes in an input dynamic graph $\mathcal{G}(t)$ by applying the following iterative procedure: 
\begin{enumerate}[leftmargin=1.2cm, align=left]
\item[\emph{Initialization}:] The colors of all nodes in $\mathcal{G}(t)$ are initialized using the initial node features: $\forall v \in V(\mathcal{G}(t)),  c^{0}(v) = x_v$. If node features are not available, all nodes receive identical colors; \vspace{-4pt} 
\item[\emph{Refinement}:] At step $\ell$, the colors of all nodes are refined using a hash (injective) function: for all $v \in V(\mathcal{G}(t))$, we apply $c^{\ell+1}(v) = \textsc{Hash}(c^{\ell}(v), \multiset{(c^{\ell}(u), e_{u v}(t^\prime), t^\prime): (u, v, t^\prime) \in \mathcal{G}(t)})$; \vspace{-4pt}
\item[\emph{Termination}:] The test is carried out for two temporal graphs at time $t$ in parallel and stops when the multisets of corresponding colors diverge, returning non-isomorphic. If the algorithm runs until the number of different colors stops increasing, the test is deemed inconclusive.
\end{enumerate}

We note that the temporal WL test trivially reduces to the standard 1-WL test if all timestamps and edge features are identical. The resemblance between \mptgns and GNNs and their corresponding WL tests suggests that the power of \mptgns is bounded by the temporal WL test. \autoref{cor:temporal-wl-tgns} conveys that \mptgns with injective layers are as powerful as the temporal WL test.

\begin{proposition} \label{cor:temporal-wl-tgns} Assume finite spaces of initial node features $\mathcal{X}$, edge features $\mathcal{E}$, and timestamps $\mathcal{T}$. Let the number of events of any temporal graph be bounded by a fixed constant. Then, there is an \mptgn with suitable parameters using injective aggregation/update functions that outputs different representations for two temporal graphs if and only if the temporal-WL test outputs `non-isomorphic'. 
\end{proposition}

A natural consequence of the limited power of \mptgns is that even the most powerful \mptgns fail to distinguish relevant graph properties, and the same applies to CAWs (see \autoref{cor:properties}).

\begin{proposition}\label{cor:properties}
There exist non-isomorphic temporal graphs differing in properties such as diameter, girth, and total number of cycles, which cannot be differentiated by \mptgns and CAWs.
\end{proposition}

\captionsetup[figure]{font=small}
\begin{wrapfigure}[8]{r}{0.45\textwidth}
\vspace{-10pt}
    \centering
    \includegraphics[width=0.43\textwidth]{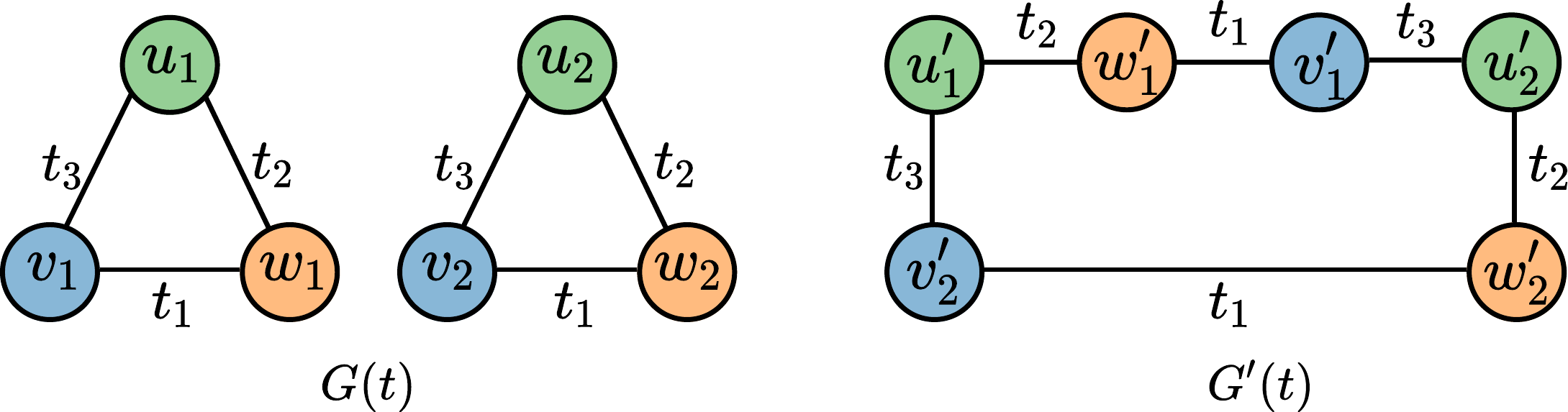}
    \caption{\small Examples of temporal graphs for which \mptgns cannot distinguish the diameter, girth, and number of cycles.}
    \label{fig:properties-main}
\end{wrapfigure}
\captionsetup[figure]{font=normal}
\autoref{fig:properties-main} provides a construction for \autoref{cor:properties}. The temporal graphs $\mathcal{G}(t)$ and $\mathcal{G}^\prime(t)$ differ in diameter ($\infty$ vs. 3), girth (3 vs. 6), and number of cycles (2 vs. 1). By inspecting the TCTs, one can observe that, for any node in $\mathcal{G}(t)$, there is a corresponding one in $\mathcal{G}^\prime(t)$ whose TCTs are isomorphic, e.g., $T_{u_1}(t) \cong T_{u_1^\prime}(t)$ for $t>t_3$. As a result, the multisets of node embeddings for these temporal graphs are identical. We provide more details and a construction - where CAW fails to decide properties - in the Appendix.

\section{\namecamelcase}
We now leverage insights from our analysis in Section 3 to build more powerful TGNs.
First, we discuss how to build injective aggregation and update functions in the temporal setting. Second, we propose an efficient scheme to compute positional features based on counts from TCTs.
In addition, we show that the proposed method, called \textit{\textbf{p}osition-encoding \textbf{in}jective \textbf{t}emporal graph net} (\initials), is more powerful than both \watgns and \mptgns in distinguishing events in temporal graphs.

\noindent \textbf{Injective temporal aggregation.}
An important design principle in TGNs is to prioritize (give higher importance to) events based on recency \citep{tgat,caw}. 
\autoref{prop:injective-temporal-aggregation} introduces an injective aggregation scheme that captures this principle using linearly exponential time decay.

\begin{proposition}[Injective function on temporal neighborhood]\label{prop:injective-temporal-aggregation}
Let $\mathcal{X}$ and $\mathcal{E}$ be countable, and $\mathcal{T}$ countable and bounded. There exists a function $f$ and scalars $\alpha$ and $\beta$ such that $\sum_i f(x_i, e_i) \alpha^{-\beta t_i}$ is unique on any multiset  $M=\{\!\!\{(x_i, e_i, t_i)\}\!\!\} \subseteq \mathcal{X} \times \mathcal{E} \times \mathcal{T}$ with $|M|<N$, where $N$ is a constant.
\end{proposition}

Leveraging \autoref{prop:injective-temporal-aggregation} and the approximation capabilities of multi-layer perceptrons (MLPs), we propose \emph{\name} (\initials). In particular, \initials computes the embedding of node $v$ at time $t$ and layer $\ell$ using the following message passing steps:  
\begin{align}
    \tilde{h}_v^{(\ell)}(t) & = \sum_{(u, e, t^\prime) \in \mathcal{N}(v, t)} \textsc{mlp}_{\text{agg}}^{(\ell)}\left(h_u^{(\ell-1)}(t) \mathbin\Vert e \right)  \alpha^{-\beta(t-t^\prime)} \label{eq:sinet_aggregation} \\
    h_v^{(\ell)}(t) & = \textsc{mlp}_{\text{upd}}^{(\ell)}\left(h_v^{(\ell-1)}(t) \mathbin\Vert  \tilde{h}_v^{(\ell)}(t)\right)  \label{eq:sinet_update}
\end{align}
where $\|$ denotes concatenation, $h_v^{(0)}=s_v(t)$, $\alpha$ and $\beta$ are scalar (hyper-)parameters, and $\textsc{MLP}_{\text{agg}}^{(\ell)}$ and $\textsc{MLP}_{\text{upd}}^{(\ell)}$ denote the nonlinear transformations of the aggregation and update steps, respectively.

We note that to guarantee that the MLPs in \initials implement injective aggregation/update, we must further assume that the edge and node features (states) take values in a finite support. In addition, we highlight that there may exist many other ways to achieve injective temporal MP --- we have presented a solution that captures the `recency' inductive bias of real-world temporal networks.

\noindent \textbf{Relative positional features.} To boost the power of PINT, we propose augmenting memory states with \emph{relative} positional features. These features count how many temporal walks of a given length exist between two nodes, or equivalently, how many times nodes appear at different levels of TCTs. 

Formally, let $P$ be the $d \times d$ matrix obtained by padding a $(d-1)$-dimensional identity matrix with zeros on its top row and its rightmost column. Also, let ${r}_{j \rightarrow u}^{(t)} \in \mathbb{N}^{d}$ denote the positional feature vector of node $j$ relative to $u$'s TCT at time $t$. For each event $(u, v, t)$, with $u$ and $v$ not participating in other events at $t$, we recursively update the positional feature vectors as
\begin{minipage}{.5\linewidth}
\begin{align}
 &\mathcal{V}_i^{(0)}  = \{i\}  \quad \forall i  \label{eq:positional1} \\
 &r^{(0)}_{i \rightarrow j} = 
    \begin{dcases}
    [1, 0, \ldots, 0]^\top  & \text{ if } i = j\\
    [0,  0, \ldots, 0]^\top  & \text{ if } i \neq j
    \end{dcases} \label{eq:positional2}
\end{align}
\end{minipage}%
\begin{minipage}{.5\linewidth}
\begin{align}
 &\mathcal{V}_{u}^{(t^+)} = \mathcal{V}_{v}^{(t^+)} = \mathcal{V}_{v}^{(t)} \cup \mathcal{V}_{u}^{(t)} \label{eq:positional3}\\
&{r}_{i \rightarrow v}^{(t^+)} = P~{r}_{i \rightarrow u}^{(t)} + {r}_{i \rightarrow v}^{(t)} \quad \forall i \in \mathcal{V}_u^{(t)} \label{eq:positional4}\\
&{r}_{j \rightarrow u}^{(t^+)} = P~{r}_{j \rightarrow v}^{(t)} + {r}_{j \rightarrow u}^{(t)} \quad \forall j \in \mathcal{V}_v^{(t)} \label{eq:positional5}
\end{align}
\end{minipage}

where we use $t^{+}$ to denote values ``right after'' $t$. The set $\mathcal{V}_i$ keeps track of the nodes for which we need to update positional features when $i$ participates in an interaction. For simplicity, we have assumed that there are no other events involving $u$ or $v$ at time $t$. Appendix \ref{ap:proof-positional-features} provides equations for the general case where nodes can participate in multiple events at the same timestamp.

\captionsetup[figure]{font=small}
\begin{wrapfigure}[13]{r}{0.35\textwidth}
    \vspace{-12pt}
    \centering
    \includegraphics[width=0.33\textwidth]{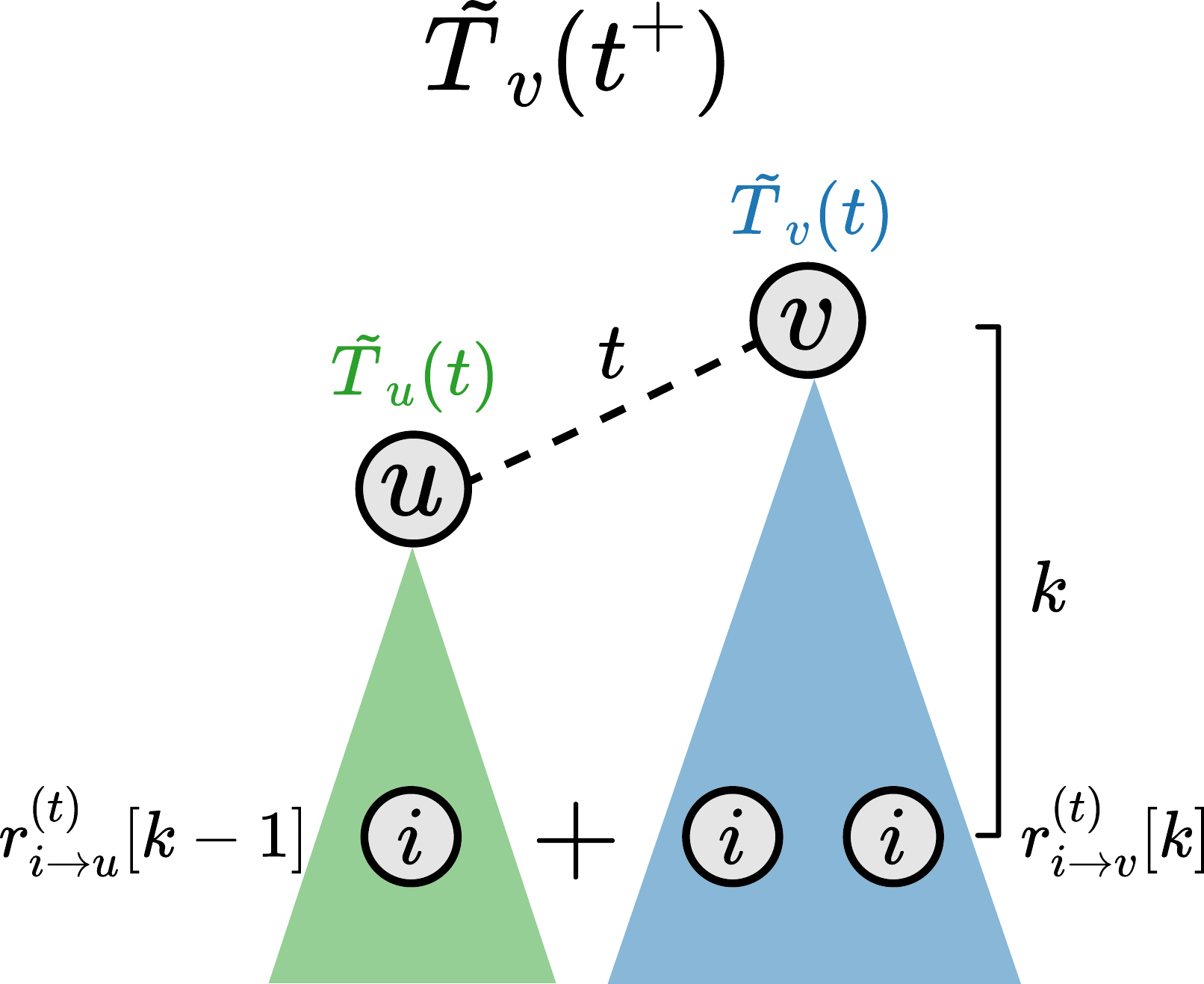}
    \caption{The effect of $(u, v, t)$ on the monotone TCT of $v$. Also, note how the positional features of a node $i$, relative to $v$, can be incrementally updated.}
    \label{fig:positional-features-main}
\end{wrapfigure}
\captionsetup[figure]{font=normal}
The value ${r}_{i \rightarrow v}^{(t)}[k]$ (the $k$-th component of ${r}_{i \rightarrow v}^{(t)}$) corresponds to how many different ways we can get from $v$ to $i$ in $k$ steps through temporal walks. Additionally, we provide in \autoref{lemma:positional-features} an interpretation of relative positional features in terms of the so-called {\tcTCT}s (\autoref{def:monotone_tct}). 
In this regard, \autoref{fig:positional-features-main} shows how the TCT of $v$ evolves due to an event $(u, v, t)$ and provides an intuition about the updates in Eqs. 10-11. The procedure amounts to appending the \tcTCT of $u$ to the first level of the monotone TCT of $v$. 

\begin{definition}\label{def:monotone_tct}
The \tcTCT of a node $u$ at time $t$, denoted by $\tilde{T}_u(t)$, is the maximal subtree of the TCT of $u$ s.t. for any path $p=(u, t_1, u_1, t_2, u_2, \dots)$ from the root $u$ to leaf nodes of $\tilde{T}_u(t)$ time monotonically decreases, i.e., we have that
$t_{1} > t_{2} > \dots $.
\end{definition}

\begin{lemma}\label{lemma:positional-features}
For any pair of nodes $i, u$ of a temporal graph $\mathcal{G}(t)$, the $k$-th component of the positional feature vector ${r}_{i \rightarrow u}^{(t)}$ stores the number of times $i$  appears at the $k$-th layer of the \tcTCT of $u$.
\end{lemma}

\begin{figure*}[t]
    \centering
    \includegraphics[width=\textwidth]{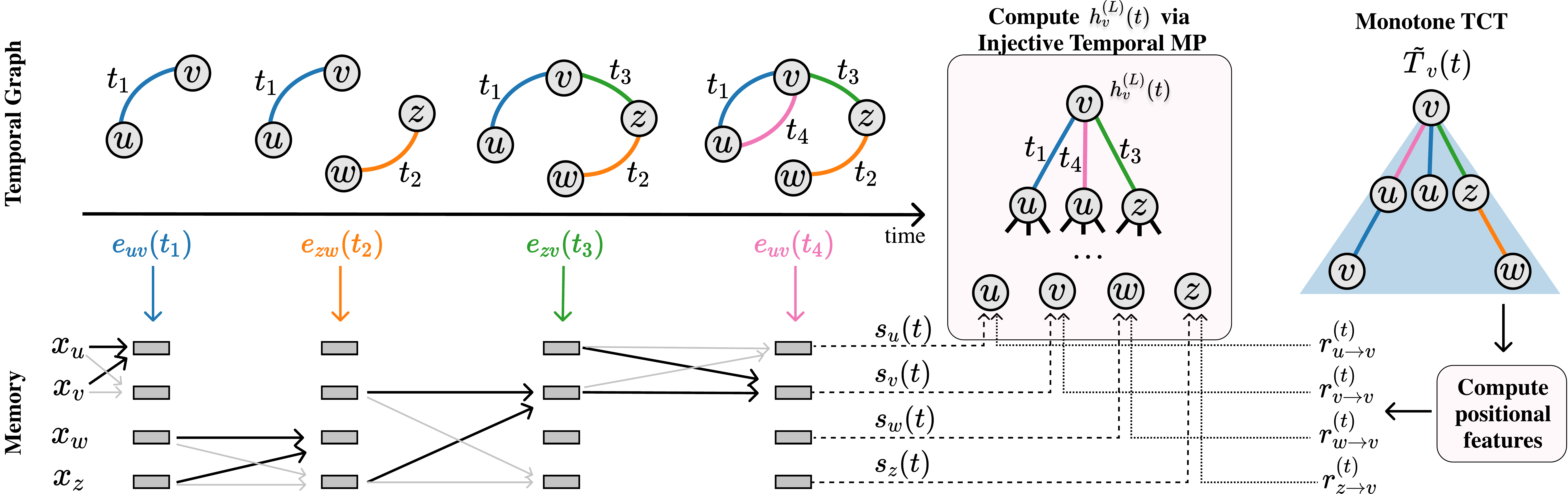}
    \caption{{\bf \initials.} Following the \mptgn protocol, PINT updates memory states as events unroll. Meanwhile, we use Eqs. (\ref{eq:positional1}-\ref{eq:positional5}) to update positional features. To extract the embedding for node $v$, we build its TCT, annotate nodes with memory + positional features, and run (injective) MP.}
    \label{fig:pint}
\vspace{-0.35cm}
\end{figure*}

\paragraph{Edge and node embeddings.} To obtain the embedding $h_\gamma$ for an event $\gamma=(u, v, t)$, an $L$-layer \initials computes embeddings for node $u$ and $v$ using $L$ steps of temporal message passing. However, when computing the embedding $h^{L}_u(t)$ of $u$, we concatenate node states $s_j(t)$ with the positional features $r_{j \rightarrow u}^{(t)}$ and $r_{j \rightarrow v}^{(t)}$ for all node $j$ in the $L$-hop temporal neighborhood of $u$. We apply the same procedure to obtain $h^{L}_{v}(t)$, and then combine $h^{L}_{v}(t)$ and $h^{L}_{u}(t)$ using a readout function.
%

Similarly, to compute representations for node-level prediction, for each node $j$ in the $L$-hop neighborhood of $u$, we concatenate node states $s_j(t)$ with features $r_{j \rightarrow u}^{(t)}$. Then, we use our injective MP to combine the information stored in $u$ and its neighboring nodes. \autoref{fig:pint} illustrates the process. 

Notably, \autoref{prop:expressiveness_sinet} states that \initials is strictly more powerful than existing TGNs. In fact, the relative positional features mimic the discriminative power of \watgns, while eliminate their temporal monotonicity constraints. Additionally, \initials can implement injective temporal message passing (either over states or states + positional features), akin to maximally-expressive \mptgns.  
\begin{proposition}
[Expressiveness of \initials: link prediction] \initials (with relative positional features) is strictly more powerful than \mptgns and CAWs in distinguishing events in temporal graphs. \label{prop:expressiveness_sinet}
\end{proposition}

\captionsetup[figure]{font=small}
\begin{wrapfigure}[9]{r}{0.22\textwidth}
\vspace{-13pt}
    \centering
    \includegraphics[width=0.15\textwidth]{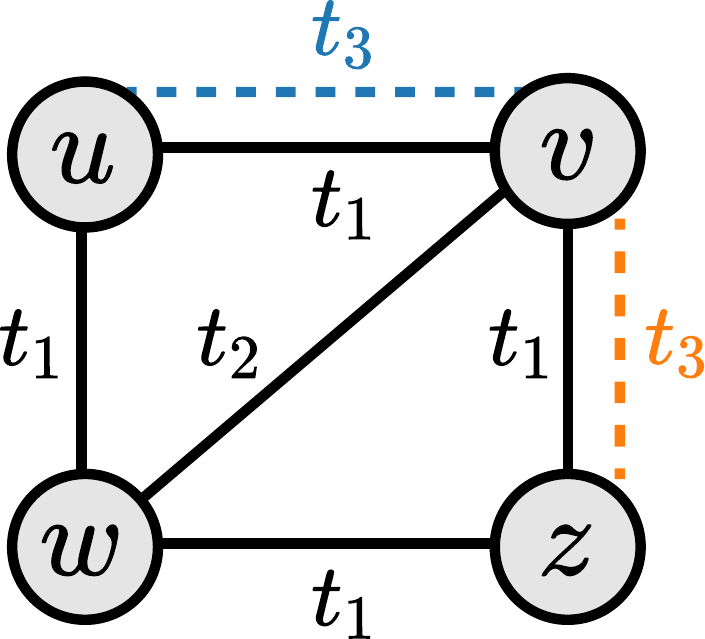}
    \caption{\small \initials cannot distinguish the events $(u, v, t_3)$ and $(v, z, t_3)$.}
    \label{fig:limitation_ourmethod}
\end{wrapfigure}
\captionsetup[figure]{font=normal}
\paragraph{When does \initials fail?} Naturally, whenever the TCTs (annotated with positional features) for the endpoints of two edges $(u, v, t)$ and $(u^\prime, v^\prime, t)$ are pairwise isomorphic, \initials returns the same edge embedding and is not able to differentiate the events. \autoref{fig:limitation_ourmethod} shows an example in which this happens --- we assume that all node/edge features are identical. Due to graph symmetries, $u$ and $z$ occur the same number of times in each level of $v$'s \tcTCT. Also, the sets of temporal walks starting at $u$ and $z$ are identical if we swap the labels of these nodes. Importantly, CAWs and \mptgns also fail here, as stated in \autoref{prop:expressiveness_sinet}. 

\begin{proposition}[Limitations of \initials]\label{prop:failure}
There are synchronous events of temporal graphs that \initials cannot distinguish (as seen in \autoref{fig:limitation_ourmethod}).
\end{proposition}
 
\paragraph{Implementation and computational cost.} The online updates for {\initials}'s positional features have complexity $\mathcal{O}\left(d\, |\mathcal{V}_u^{(t)}| + d\, |\mathcal{V}_v^{(t)}|\right)$. 
Similarly to CAW's sampling procedure, our online update is a sequential process better done in CPUs. However, while CAW may require significant CPU-GPU memory exchange --- proportional to both the number of walks and their depth ---, we only communicate the positional features. We can also speed-up the training of \initials by pre-computing the positional features for each batch, avoiding redundant computations at each epoch. 
Apart from positional features, the computational cost of \initials is similar to that of TGN-Att. 
Following standard \mptgn procedure, we control the branching factor of TCTs using neighborhood sampling.

Note that the positional features monotonically increase with time, which is undesirable for practical generalization purposes. 
Since our theoretical results hold for any fixed $t$, this issue can be solved by dividing the positional features by a time-dependent normalization factor. Nonetheless, we have found that employing $L_1$-normalization leads to good empirical results for all evaluated datasets.

\section{Experiments}
We now assess the performance of \initials on several popular and large-scale benchmarks for TGNs.
We run experiments using PyTorch \citep{pytorch} and code is available at \url{www.github.com/AaltoPML/PINT}.

\paragraph{Tasks and datasets.} We evaluate \initials on dynamic link prediction, closely following the evaluation setup employed by \citet{tgn} and \citet{tgat}. 
We use six popular benchmark datasets: Reddit, Wikipedia, Twitter, UCI, Enron, and LastFM \citep{jodie,tgn,tgat,caw}. 
Notably, UCI, Enron, and LastFM are non-attributed networks, i.e., they do not contain feature vectors associated with the events. 
Node features are absent in all datasets, thus following previous works we set them to vectors of zeros \citep{tgn,tgat}.
Since Twitter is not publicly available, we follow the guidelines by \citet{tgn} to create our version. 
We provide more details regarding datasets in the supplementary material.

\begin{table}[th]
\caption{\small \textbf{Average Precision} (AP) results for link prediction. We denote the best-performing model (highest mean AP) in \textcolor{lb}{\bf blue}. In 5 out of 6 datasets, \initials achieves the highest AP in the transductive setting. For the inductive case, \initials outperforms previous \mptgns and competes with CAW. We also show the performance of PINT with and without relative positional features. For all datasets, adopting positional features leads to significant performance gains.}
\centering
\resizebox{\textwidth}{!}{
\begin{tabular}{l l c c c c c c|}
\toprule
& \textbf{Model} & \textbf{Reddit} & \textbf{Wikipedia} & \textbf{Twitter} & \textbf{UCI} & \textbf{Enron} & \textbf{LastFM} \\ 
\cmidrule{2-8}
 \parbox[t]{2mm}{\multirow{9}{*}{\rotatebox[origin=c]{90}{\bf Transductive}}} & GAT & $97.33   \pm 0.2 $ & $94.73   \pm 0.2 $ & - & -  & - &  - \\
 & GraphSAGE & $97.65  \pm 0.2 $ & $93.56 \pm 0.3$ & - & -  & - &  -\\
 \cmidrule{2-8}
& Jodie & $97.11 \pm 0.3$ & $94.62 \pm 0.5$ & $98.23 \pm 0.1$ & {$86.73 \pm 1.0$} & $77.31 \pm 4.2$ & $69.32 \pm 1.0$\\
& DyRep & $97.98 \pm 0.1$ & $94.59 \pm 0.2$ & $98.48 \pm 0.1$ & $54.60 \pm 3.1$ & $77.68 \pm 1.6$ & $69.24 \pm 1.4$\\
& TGAT & $98.12 \pm 0.2$ & $95.34 \pm 0.1$ & ${98.70 \pm 0.1}$ & $77.51 \pm 0.7$ & $68.02 \pm 0.1$ & $54.77 \pm 0.4$\\
& TGN-Att & ${98.70 \pm 0.1}$ & ${98.46 \pm 0.1}$ & $98.00 \pm 0.1$ & $80.40 \pm 1.4$ & {$79.91 \pm 1.3$} & {$80.69 \pm 0.2$}\\
& CAW & $98.39 \pm 0.1$ &  ${98.63 \pm 0.1}$  & $98.72 \pm 0.1$
&
$92.16 \pm 0.1$
&
$\textcolor{lb}{\bf 92.09 \pm 0.7}$
&
$81.29 \pm 0.1$ \\
 \cmidrule{2-8}
& {\bf\initials} {\scriptsize (w/o pos. feat.)} & {$98.62 \pm .04$} & {$ 98.43 \pm .04$} & {$98.53 \pm 0.1$} & ${92.68 \pm 0.5}$ & ${83.06 \pm 2.1}$ & ${81.35 \pm 1.6}$ \\
& {\bf\initials}  & \textcolor{lb}{$\bf 99.03 \pm .01$} 
& $\textcolor{lb}{\bf 98.78 \pm 0.1}$ & $\textcolor{lb}{\bf 99.35 \pm .01}$ & $\textcolor{lb}{\bf 96.01 \pm 0.1}$ & $88.71 \pm 1.3$ & $\textcolor{lb}{\bf 88.06 \pm 0.7}$\\
\midrule
\midrule
\parbox[t]{2mm}{\multirow{9}{*}{\rotatebox[origin=c]{90}{\bf Inductive}}}  &   GAT & $95.37    \pm 0.3 $ & $91.27  \pm 0.4 $ & - & -  & - &  - \\
   &  GraphSAGE & $96.27   \pm 0.2 $ & $91.09 \pm 0.3$ & - & -  & - &  -\\
 \cmidrule{2-8}
&    Jodie & $94.36 \pm 1.1$ & $93.11 \pm 0.4$ & $96.06 \pm 0.1$ & {$75.26 \pm 1.7$} & {$76.48 \pm 3.5$} & $80.32 \pm 1.4$\\
 &   DyRep & $95.68 \pm 0.2$ & $92.05 \pm 0.3$ & ${96.33 \pm 0.2}$ & $50.96 \pm 1.9$ & $66.97 \pm 3.8$ & $82.03 \pm 0.6$\\
  &  TGAT & $96.62  \pm 0.3$ & $93.99 \pm 0.3$ & ${96.33 \pm 0.1}$ & $70.54 \pm 0.5$ & $63.70 \pm 0.2$ & $56.76 \pm 0.9$\\
   & TGN-Att & ${97.55 \pm 0.1}$  & ${97.81 \pm 0.1}$ & $95.76 \pm 0.1$ & $74.70 \pm 0.9$ & ${78.96 \pm 0.5}$ & {$84.66 \pm 0.1$}\\
& CAW & $97.81 \pm 0.1$
& \textcolor{lb}{$\bf 98.52 \pm 0.1$} &
\textcolor{lb}{$\bf 98.54 \pm 0.4$} & $92.56 \pm 0.1$ &  \textcolor{lb}{$\bf 91.74 \pm 1.7$} &  $85.67 \pm 0.5$ \\
 \cmidrule{2-8}
  &  {\bf\initials} {\scriptsize (w/o pos. feat.)} & {$97.22 \pm 0.2$} & ${97.81 \pm 0.1}$ & $96.10 \pm 0.1$ & ${90.25 \pm 0.3}$ & $75.99 \pm 2.3$ & ${88.44 \pm 1.1}$\\
   & {\bf \initials}  & \textcolor{lb}{$\bf 98.25 \pm .04$}  & $98.38 \pm .04$  &  $98.20 \pm .03$ &  \textcolor{lb}{$\bf93.97 \pm 0.1$} & $81.05 \pm 2.4$ & \textcolor{lb}{$\bf91.76 \pm 0.7$} \\
    \bottomrule
\end{tabular}}
\label{tab:transductive}
\vspace{-0.35cm}
\end{table}

\paragraph{Baselines.} We compare \initials against five prominent TGNs: Jodie \citep{jodie}, DyRep \citep{DyRep}, TGAT \citep{tgat}, TGN-Att \citep{tgn}, and CAW \citep{caw}. For completeness, we also report results using two static GNNs: GAT \citep{gat} and GraphSage \citep{graphsage}. 
Since we adopt the same setup as TGN-Att, we use their table numbers for all baselines but CAW on Wikipedia and Reddit.
The remaining results were obtained using the implementations and guidelines available from the official repositories.
As an ablation study, we also include a version of \initials without relative positional features in the comparison.
We provide detailed information about hyperparameters and the training of each model in the supplementary material.

\paragraph{Experimental setup.} We follow \citet{tgat} and use a 70\%-15\%-15\% (train-val-test) temporal split for all datasets.
We adopt average precision (AP) as the performance metric.
We also analyze separately predictions involving only nodes seen during training (transductive), and those involving novel nodes (inductive).
We report mean and standard deviation of the AP over ten  runs.
For further details, see \autoref{ap:implementation}.
We provide additional results in the supplementary material.

\textbf{Results.} \autoref{tab:transductive} shows that \initials is the best-performing method on five out of six datasets for the transductive setting. 
Notably, the performance gap between \initials and TGN-Att amounts to over 15\% AP on UCI.
The gap is also relatively high compared to CAW on LastFM, Enron, and UCI; with CAW being the best model only on Enron.
We also observe that many models achieve relatively high AP on the attributed networks (Reddit, Wikipedia, and Twitter).
This aligns well with findings from \citep{caw}, where TGN-Att was shown to have competitive performance against CAW on Wikipedia and Reddit.
The performance of GAT and TGAT (static GNNs) on Reddit and Wikipedia reinforces the hypothesis that the edge features add significantly to the discriminative power. On the other hand, \initials and CAW, which leverage relative identities, show superior performance relative to other methods when only time and degree information is available, i.e., on unattributed networks (UCI, Enron, and LastFM). \autoref{tab:transductive} also shows the effect of using relative positional features. %
While including these features boosts PINT's performance systematically, our ablation study shows that \initials w/o positional features still outperforms other \mptgns on unattributed networks.
In the inductive case, we observe a similar behavior: \initials is consistently the best \mptgn, and is better than CAW on 3/6 datasets. 
Overall, PINT (w/ positional features) also yields the lowest standard deviations. This suggests that positional encodings might be a useful inductive bias for TGNs.

\captionsetup[figure]{font=small}
\begin{wrapfigure}[13]{r}{0.55\textwidth}
\vspace{-14pt}
    \centering
    \includegraphics[width=0.53\textwidth]{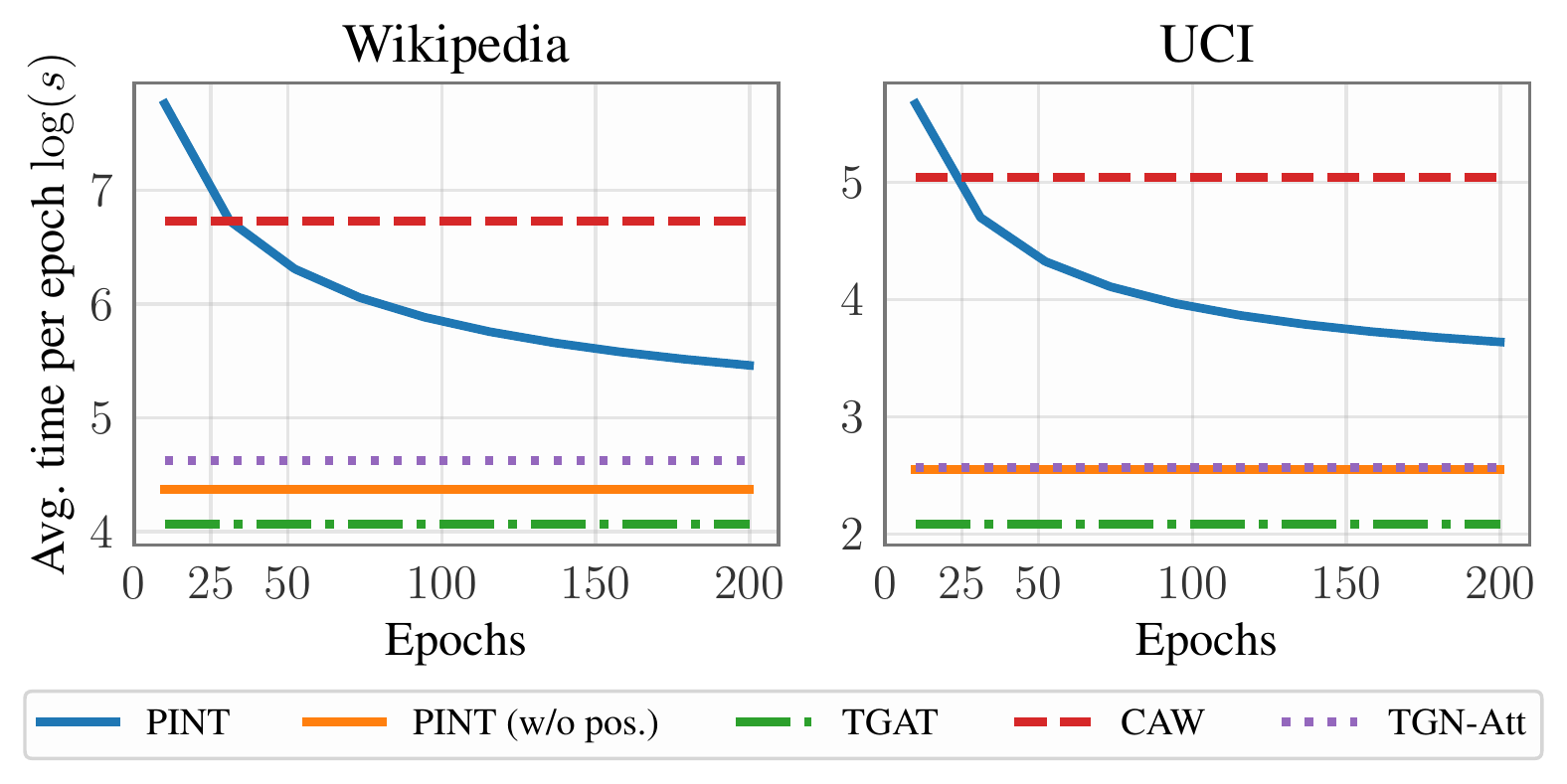}
    \caption{\small Time comparison: \initials versus TGNs  (in log-scale). The cost of pre-computing positional features is quickly diluted as the number of epochs increases.}
    \label{fig:time-comparison}
\end{wrapfigure}
\captionsetup[figure]{font=normal}
\textbf{Time comparison.} \autoref{fig:time-comparison} compares the training times of \initials against other TGNs. For fairness, we use the same architecture (number of layers \& neighbors) for all \mptgns: i.e., the best-performing \initials. For CAW, we use the one that yielded  results in \autoref{tab:transductive}. As expected, TGAT is the fastest model. Note that the average time/epoch of \initials gets amortized since positional features are pre-computed. Without these features, {\initials}'s runtime closely matches TGN-Att. When trained for over $25$ epochs, \initials runs considerably faster than CAW. We provide additional details and results in the supplementary material.

\paragraph{Incorporating relative positional features into MP-TGNs.} We can use our relative positional features (RPF) to boost MP-TGNs. 
\autoref{tab:tgn_pe} shows the performance of TGN-Att with relative positional features on UCI, Enron, and LastFM. Notably, TGN-Att receives a significant boost from our RPF. However, PINT still beats TGN-Att+RPF on 5 out of 6 cases. The values for TGN-Att+RPF reflect outcomes from $5$ repetitions.
We have used the same model selection procedure as TGN-Att in Table 1, and incorporated $d=4$-dimensional positional features

\begin{table}[t]
\centering
\caption{Average precision results for TGN-Att + relative positional features.}
\resizebox{0.85\textwidth}{!}{
\begin{tabular}{@{\extracolsep{4pt}}lcccccc}
\toprule 
& \multicolumn{3}{c}{\textbf{Transductive}} & \multicolumn{3}{c}{\textbf{Inductive}} \\ \cmidrule{2-4}  \cmidrule{5-7} 
& {UCI} & {Enron} &   {LastFM} & {UCI} & {Enron} &   {LastFM}  \\ \midrule
TGN-Att & $80.40 \pm 1.4$ & $79.91 \pm 1.3$ & $80.69 \pm 0.2$  & $74.70 \pm 0.9$ & $78.96 \pm 0.5$ & $84.66 \pm 0.1$\\
TGN-Att + RPF & $95.64 \pm 0.1$ & $85.04 \pm 2.5$ & $89.41 \pm 0.9$  & $92.82 \pm 0.4$ & $76.27 \pm 3.4 $ & $91.63 \pm 0.3$\\
PINT & $96.01 \pm 0.1$ & $88.71 \pm 1.3$ & $88.06 \pm 0.7$ & $93.97 \pm 0.1$ & $81.05 \pm 2.4$ & $91.76 \pm 0.7$\\
\bottomrule
\end{tabular}
}
\label{tab:tgn_pe}
\end{table}

\captionsetup[figure]{font=small}
\begin{wrapfigure}[11]{r}{0.5\textwidth}
\vspace{-14pt}
    \centering
    \includegraphics[width=0.5\textwidth]{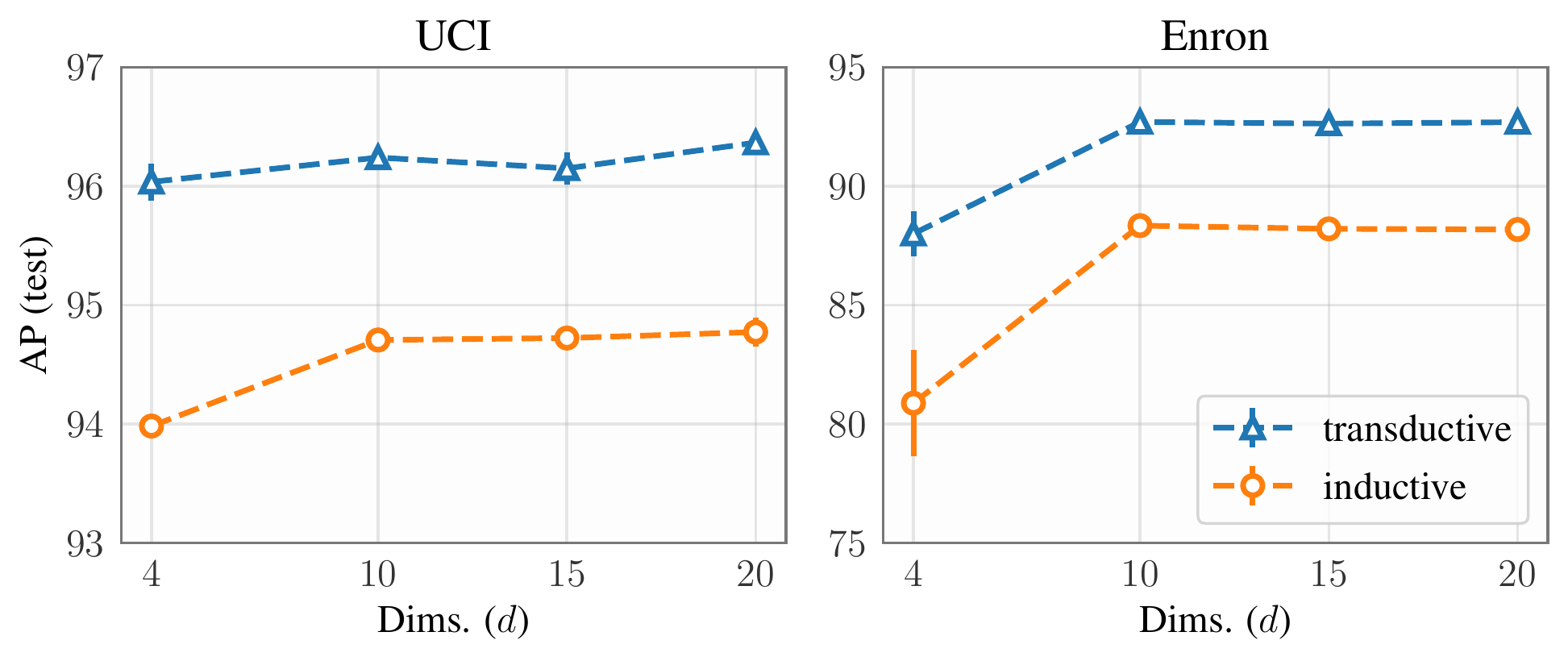}
    \caption{\initials: AP (mean and std) as a function of the dimensionality of the positional features.}
    \label{fig:dims-pos-features}
\end{wrapfigure}
\captionsetup[figure]{font=normal}
\paragraph{Dimensionality of relative positional features.} We assess the performance of \initials as a function of the dimension $d$ of the relative positional features. \autoref{fig:dims-pos-features} shows the performance of \initials for $d \in \{4, 10, 15, 20\}$ on UCI and Enron. We report mean and standard deviation of AP on test set obtained from five independent runs. In all experiments, we re-use the optimal hyper-parameters found with $d=4$. Increasing the dimensionality of the positional features leads to performance gains on both datasets. Notably, we obtain a significant boost for Enron with $d=10$: $92.69 \pm 0.09$ AP in the transductive setting and $88.34 \pm 0.29$ in the inductive case. Thus, \initials becomes the best-performing model on Enron (transductive). On UCI, for $d=20$, we obtain $96.36 \pm 0.07$ and $94.77 \pm 0.12$ (inductive).

\section{Conclusion}

We laid a rigorous theoretical foundation for TGNs, 
including the role of memory modules, relationship between classes of TGNs, and failure cases for \mptgns. Together, our theoretical results shed light on the representational capabilities of TGNs, and connections with their static counterparts. We also introduced a novel TGN method, provably more expressive than the existing TGNs. 

Key practical takeaways from this work: (a) temporal models should be designed to have injective update rules and to exploit both neighborhood and walk aggregation, and (b) deep architectures can likely be made more compute-friendly as the role of memory gets diminished with depth, provably.    

\begin{ack}
This work was supported by the Academy of Finland (Flagship programme: Finnish Center for Artificial Intelligence FCAI and 341763), ELISE Network of Excellence Centres (EU Horizon:2020 grant agreement 951847) and UKRI Turing AI World-Leading Researcher Fellowship, EP/W002973/1. We also acknowledge the computational resources provided by the Aalto Science-IT Project from Computer Science IT.
AS and DM also would like to thank Jorge Perez, Jou-Hui Ho, and Hojin Kang for valuable discussions about TGNs, and the latter's input on a preliminary version of this work. 
\end{ack}

\section*{Societal and broader impact}
Temporal graph networks have shown remarkable performance in relevant domains such as social networks, e-commerce, and drug discovery. In this paper, we establish fundamental results that delineate the representational power of TGNs. We expect that our findings will help declutter the literature and serve as a seed for future developments. Moreover, our analysis culminates with \initials, a method that is provably more powerful than the prior art and shows superior predictive performance on several benchmarks. We believe that \initials (and its underlying concepts) will help engineers and researchers build better recommendation engines, improving the quality of systems that permeate our lives. Also, we do not foresee any negative societal impact stemming directly from this work.




{
\small
\bibliographystyle{plainnat} 
\bibliography{references}
}

\newpage

\clearpage


\makeatletter
  \setcounter{table}{0}
  \renewcommand{\thetable}{S\arabic{table}}%
  \setcounter{figure}{0}
  \renewcommand{\thefigure}{S\arabic{figure}}%
  \setcounter{equation}{0}
  \renewcommand\theequation{S\arabic{equation}}
  \renewcommand{\bibnumfmt}[1]{[#1]}

  \newcommand{\suptitle}{Provably expressive temporal graph networks \\ (Supplementary material)}
  \renewcommand{\@title}{\suptitle}
  \renewcommand{\@author}{}

  \par
  \begingroup
    \renewcommand{\thefootnote}{\fnsymbol{footnote}}
    \renewcommand{\@makefnmark}{\hbox to \z@{$^{\@thefnmark}$\hss}}
    \renewcommand{\@makefntext}[1]{%
      \parindent 1em\noindent
      \hbox to 1.8em{\hss $\m@th ^{\@thefnmark}$}#1
    }
    \thispagestyle{empty}
    \@maketitle
  \endgroup
  \let\maketitle\relax
  \let\thanks\relax
\makeatother

\appendix

\setcounter{lemma}{0}
\setcounter{definition}{0}
\setcounter{proposition}{0}
\renewcommand{\thelemma}{\Alph{section}\arabic{lemma}}
\renewcommand{\thedefinition}{\Alph{section}\arabic{definition}}
\renewcommand{\theproposition}{\Alph{section}\arabic{proposition}}

\section{Further details on temporal graph networks}\label{ap:egn}

In this section we present more details about the models TGAT, TGN-Att, and CAW.

\subsection{Temporal graph attention (TGAT)} 
Temporal graph attention networks \citep{tgat} combine time encoders \citep{Kazemi2019} and self-attention \citep{Vaswani2017}. In particular, the time encoder $\phi$ is given by
\begin{align}\label{eq:time-encoder}
        \phi(t - t^\prime)=[\cos(\omega_1 (t - t^\prime) + b_1), \dots, \cos(\omega_d (t - t^\prime) + b_d)],
\end{align}
where $\omega_i$'s and $b_i$'s are learned scalar parameters. The time embeddings are concatenated to the edge features before being fed into a typical self-attention layer, where the query $q$ is a function of a reference node $v$, and both values $V$ and keys $K$ depend on $v$'s temporal neighbors. Formally, TGAT first computes a matrix $C^{(\ell)}_v(t)$ whose $u$-th row is $c^{(\ell)}_{vu}(t)  = [h_u^{(\ell-1)}(t) \mathbin\Vert \phi(t-t_{uv}) \mathbin\Vert e_{uv}]$ for all $(u, e_{uv}, t_{uv}) \in \mathcal{N}(v, t)$.
Then, the output $\tilde{h}_v^{(\ell)}(t)$ of the $\textsc{Agg}^{(\ell)}$ function is given by
\begin{align}\label{eq:tgat-aggregation}
& q  = [h_v^{(\ell-1)}(t) \mathbin\Vert \phi(0)] W^{(\ell)}_q \quad K  =  C^{(\ell)}_v(t) W^{(\ell)}_K  \quad V  =  C^{(\ell)}_v(t) W^{(\ell)}_V \\
& \tilde{h}_v^{(\ell)}(t)  =  \softmax\left(q K^\top \right)V
\end{align}
where $W^{(\ell)}_q, W^{(\ell)}_K,$ and $W^{(\ell)}_V$ are model parameters.
Regarding the \textsc{Update} function, TGAT applies a multilayer perceptron, i.e., $h_v^{(\ell)}(t) = \textsc{MLP}^{(\ell)}(h_v^{(\ell-1)}(t) \mathbin\Vert \tilde{h}_v^{(\ell)}(t))$. 

\subsection{Temporal graph networks with attention (TGN-Att)}\label{ap:tgn-att}

We now discuss details regarding the \mptgn framework omitted from the main paper for simplicity.

For the sake of generality, \citet{tgn} present a formulation for \mptgns that can handle node-level events, e.g., node feature updates. These events lead to $i$) updating node memory states, and $ii$) using the time-evolving node features as additional inputs for the message-passing functions. Nonetheless, to the best of our knowledge, all relevant CTDG benchmarks comprise only edge events. Therefore, for ease of presentation, we omit node events and temporal node features from our treatment. In \autoref{ap:deletion}, we discuss how to handle node-level events.


Note that \mptgns update memory states only after an event occurs, otherwise it would incur information leakage. Unless we use the updated states to predict another event later on in the batch, this means that there might be no signal to propagate through memory modules. To get around this problem, \citet{tgn} propose updating the memory with messages coming from previous batches, and then predicting the interactions.

To speed up computations, \mptgns employ a form of batch learning where events in a same batch are aggregated. In our analysis, we assume that two events belong to the same batch only if they occur at the same timestamp. Importantly, memory aggregators allow removing ambiguity in the way the memory of a node participating in multiple events (at the same timestamp) is updated --- without memory aggregator, two events involving a given node $i$ at the same time could lead to different ways of updating the state of $i$.

Suppose the event $\gamma=(i, u, t)$ occurs. \mptgns proceed by computing a memory-message function $\textsc{MemMsg}_e$ for each endpoint of $\gamma$, i.e., 
\begin{align*}
m_{i,u}(t) &= \textsc{MemMsg}_e(s_i(t), s_u(t), t-t_{i}, e_{iu}(t))\\
m_{u,i}(t) &= \textsc{MemMsg}_e(s_u(t), s_i(t), t-t_{u}, e_{iu}(t))
\end{align*}

Following the original formulation, we assume an identity memory-message function --- simply the concatenation of the inputs, i.e., 
$
\textsc{MemMsg}_e(s_i(t), s_u(t), t-t_{i}, e_{iu}(t))=[s_i(t), s_u(t), t-t_{i}, e_{iu}(t)]
$.

Now, suppose two events $(i, u, t)$ and $(i, v, t)$ happen. \mptgns aggregate the memory-messages from these events using a function \textsc{MemAgg} to obtain a single memory message for $i$:
\begin{align*}
m_{i}(t) = \textsc{MemAgg}(m_{i,u}(t), m_{i,v}(t)) 
\end{align*}

\citet{tgn} propose non-learnable memory aggregators, such as the mean aggregator (average all memory messages for a given node), that we denote as $\textsc{MeanAgg}$ and adopt throughout our analysis. As an example, under events $(i, u, t)$ and $(i, v, t)$, the aggregated message for $i$ is $m_i(t) = 0.5([s_i(t), s_u(t), t-t_{i}, e_{iu}(t)] + [s_i(t), s_v(t), t-t_{i}, e_{iv}(t)])$.

The memory update of our query node $i$ is given by
\begin{align*}
s_i(t^+) = \textsc{MemUpdate}(s_i(t), m_{i}(t)).   
\end{align*}

Finally, we note that TGAT does not have a memory module. TGN-Att consists of the model resulting from augmenting TGAT with a GRU-based memory.

\subsection{Causal anonymous walks (CAW)}
We now provide details regarding how CAW obtains edge embeddings for a query event $\gamma=(u,v,t)$. 

A temporal walk is represented as $W=((w_1, t_1), (w_2, t_2), \dots, (w_L, t_L))$, with $t_1 > t_2 > \dots > t_L$ and $(w_{i-1}, w_i, t_i) \in \mathcal{G}(t)$ for all $i>1$. We denote by $S_u(t)$ the set of maximal temporal walks starting at $u$ of size at most $L$ obtained from the temporal graph at time $t$. Following the original paper, we drop the time dependence henceforth.

A given walk $W$ gets anonymized through replacing each element $w_i$ belonging to $W$ by a 2-element set of vectors $I_\text{CAW}(w_i; S_u, S_v)$ accounting for how many times $w_i$ appears at each position of walks in $S_u$ and $S_v$. These vectors are denoted by $g(w_i, S_u)$ and $g(w_i, S_v)$. The walk is encoded using a RNN:
\begin{align*}
    \textsc{Enc}(W; S_u, S_v) = \mathrm{RNN}([f_1(I_{\text{CAW}}(w_i; S_u, S_v)) \| f_2(t_i-t_{i-1})]_{i=1}^{L}),
\end{align*}
where $t_1=t_0=t$ and $f_1$ is
\begin{align*}
    f_1(I_\text{CAW}(w_i; S_u, S_v)) = \mathrm{MLP}(g(w_i, S_u)) + \mathrm{MLP}(g(w_i, S_v)).
\end{align*}
We note that the MLPs share parameters. The function $f_2$ is given by
\begin{align*}
    f_2(t) = [\cos(\omega_i t), \sin(\omega_1 t), \dots, \cos(\omega_d t), \sin(\omega_d t)]
\end{align*}
where $\omega_i$'s are learned parameters.

To compute the embedding $h_\gamma$ for $(u, v, t)$, CAW considers two readout functions: mean and self-attention. Finally, the final link prediction is obtained from a 2-layer MLP over $h_\gamma$.

\newpage
\section{Proofs} \label{ap:proofs}


\subsection{Further definitions and Lemmata}

\begin{definition}[Monotone walk.] An $N$-length monotone walk in a temporal graph $\mathcal{G}(t)$ is a sequence $(w_1, t_1, w_2, t_2, \ldots, w_{N+1})$ such that $t_i > t_{i+1}$ and $(w_i, w_{i+1}, t_i) \in \mathcal{G}(t)$ for all $i$.\label{def:temporal-walk}
\end{definition}

\begin{definition}[Temporal diameter.] We say the temporal diameter of a graph $\mathcal{G}(t)$ is $\Delta$ if the longest monotone walk in $\mathcal{G}(t)$ has length (i.e, number of edges) exactly $\Delta$.\label{def:diam}
\end{definition}

\begin{lemma}\label{lemma:isomorphism-time-contrained-TCT}
If the TCTs of two nodes are isomorphic, then their {\tcTCT}s (\autoref{def:monotone_tct}) are also isomorphic, i.e., $T_u(t) \cong T_v(t) \Rightarrow \tilde{T}_u(t) \cong \tilde{T}_v(t)$ for two nodes $u$ and $v$ of a dynamic graph.
\end{lemma}
\begin{proof}
Since $T_u(t) \cong T_v(t)$, we have that 
\begin{align}
& p = (u_0,t_1,u_1,t_2,u_2,\dots) \text{ from } T_u(t) \Longleftrightarrow p^\prime = (f(u_0), t_1, f(u_1), t_2, f(u_2), \dots) \text{ from } T_v(t), \nonumber \\
& \text{with } s_{u_i} = s_{f(u_i)} \text{ and } e_{u_i u_{i+1}}(t_{i+1}) =e_{f(u_i) f(u_{i+1})}(t_{i+1}) \text{ and }  k_{u_i}=k_{f(u_i)} \quad \forall i \nonumber
\end{align}
where $f: V(T_u(t)) \rightarrow V(T_v(t))$ is a bijection.

Assume that $\tilde{T}_u(t) \not\cong \tilde{T}_v(t)$. Then, either there exists a path $p_s=(u_0^\prime, t_1^\prime, u^\prime_1, t_2^\prime, \ldots)$ in $\tilde{T}_u(t)$, such that $t_{k+1}^\prime < t_k^\prime$ for all $k$ (i.e., a monotone walk), with no corresponding one in $\tilde{T}_v(t)$ or vice-versa. Without loss of generality, let us consider the former case.

We can construct the path $p^\prime_s$ in $T_v(t)$ by applying $f$ in all elements of $p_s$, i.e.,  $p_s^\prime = (f(u^\prime_0), t_1^\prime, f(u^\prime_1), t_2^\prime, \ldots)$. Note that $p^\prime_s$ is a monotone walk in $T_v(t)$. Since $\tilde{T}_{v}(t)$ is the maximal monotone subtree of $T_v(t)$, it must contain $p_s^\prime$, leading to contradiction.
\end{proof}

\begin{lemma} \label{lem:b2}
Let $\mathcal{G}(t)$ and $\mathcal{G}^\prime(t)$ be any two non-isomorphic temporal graphs. If an \mptgn obtains different multisets of node embeddings for $\mathcal{G}(t)$ and $\mathcal{G}^\prime(t)$. Then, the temporal WL test decides $\mathcal{G}(t)$ and $\mathcal{G}^\prime(t)$ are not isomorphic.
\end{lemma}
\begin{proof}
Recall \autoref{prop:memory_expressivity} shows that if an \mptgn with memory is able to distinguish two nodes, then there is a memoryless \mptgn with $\Delta$ (temporal diameter) additional layers that does the same. 
Thus, it suffices to show that if the multisets of colors from temporal WL for $\mathcal{G}(t)$ and $\mathcal{G}^\prime$(t) after $\ell$ iterations are identical, then the multisets of embeddings from the memoryless \mptgn are also identical, i.e., if $\multiset{c^{\ell}(u)}_{u \in V(\mathcal{G}(t))}=\multiset{c^{\ell}(u^\prime)}_{u^\prime \in V(\mathcal{G}^\prime(t))}$, then  $\multiset{h^{(\ell)}_u(t)}_{u \in V(\mathcal{G}(t))}=\multiset{h^{(\ell)}_{u^\prime}(t)}_{u^\prime \in V(\mathcal{G}^\prime(t))}$. To do so, we repurpose the proof of Lemma 2 in \citep{gin}.

More broadly, we show that for any two nodes of a temporal graph $\mathcal{G}(t)$, if the temporal WL returns $c^\ell(u) = c^\ell(v)$, we have that corresponding embeddings from \mptgn without memory are identical $h^{\ell}_u(t) = h^{\ell}_v(t)$. We proceed with a proof by induction.

[\textit{Base case}] For $\ell=0$, the proposition trivially holds as  the temporal WL has the initial node features as colors, and memoryless \mptgns have these features as embeddings.

[\textit{Induction step}] Assume the proposition holds for iteration $\ell$. Thus, for any two nodes $u, v$, if $c^{\ell+1}(u) = c^{\ell+1}(v)$, we have
\begin{align*}
(c^{\ell}(u), \multiset{(c^{\ell}(i), e_{i u}(t^\prime), t^\prime): (u, i, t^\prime) \in \mathcal{G}(t)})
= (c^{\ell}(v), \multiset{(c^{\ell}(j), e_{j v}(t^\prime), t^\prime): (v, j, t^\prime) \in \mathcal{G}(t)})   
\end{align*}
and, by the induction hypothesis, we know
\begin{align*}
\begin{split}
(h^{(\ell)}_u(t), \multiset{(h_i^{(\ell)}(t), e_{i u}(t^\prime), t^\prime): (u, i, t^\prime) \in\, &\mathcal{G}(t)})
=\\ &(h^{(\ell)}_v(t), \multiset{(h_j^{(\ell)}(t), e_{j v}(t^\prime), t^\prime): (v, j, t^\prime) \in \mathcal{G}(t)})   
\end{split}
\end{align*}

We also note that this last identity also implies
\begin{align*} 
\begin{split}
(h^{(\ell)}_u(t), \multiset{(h_i^{(\ell)}(t), t-t^\prime, e) \mid (i, e, t^\prime) \in\,  & \mathcal{N}(u, t)}) 
= \\&(h^{(\ell)}_v(t), \multiset{(h_j^{(\ell)}(t), t-t^\prime, e) \mid (j, e, t^\prime) \in \mathcal{N}(v, t)}) 
\end{split}
\end{align*}
since there exists an event $(u, i, t^\prime) \in \mathcal{G}(t)$ with feature $e_{ui}(t^\prime) = e$ iff there is an element $(i, e, t^\prime) \in \mathcal{N}(u, t)$.

As a result, the inputs of the \mptgn's aggregation and update functions are identical, which leads to identical outputs $h^{(\ell+1)}_u(t)=h^{(\ell+1)}_v(t)$. Therefore, if the temporal WL test obtains identical multisets of colors for two temporal graphs after $\ell$ steps, the multisets of embeddings at layer $\ell$ for these graphs are also identical.
\end{proof}

\begin{lemma}[Lemma 5 in \citep{gin}]\label{lemma:lemma5_gin}
Assume $\mathcal{X}$ is countable. There exists a function $f: \mathcal{X} \rightarrow \mathbb{R}^n$ so that $h(X) = \sum_{x \in X} f(x)$ is unique for each multiset $X \subset \mathcal{X}$ of bounded size. Moreover, any multiset function $g$ can be decomposed as $g(X)=\varphi\left(\sum_{x \in X} f(x)\right)$ for some function $\varphi$.
\end{lemma}

\subsection{Proof of \autoref{prop:equivalence}: Relationship between DTDGs and CTDGs}

\begin{proof} We prove the two statements in \autoref{prop:equivalence} separately.
In the following, we treat CTDGs as sets of events up to a given timestamp.

\textbf{\textcolor{lb}{Statement 1:}} For any DTDG we can build a CTDG that contains the same information.

A DTDG consists of a sequence of graphs with no temporal information. We can model this using the CTDG formalism by setting a fixed time difference $\delta$ between consecutive elements $\mathsf{G}(t_i), \mathsf{G}(t_{i+1})$ of the CTDG, i.e., $t_{i+1} - t_i = \delta$ for all $i \geq 0$. 

Consider a DTDG given by the sequence $(G_{1}, G_{2},\dots)$. To build the equivalent CTDG, we define $S(G_i)$ as the set of edge events corresponding to $G_i$, i.e., $S(G_i) = \{(u, v, i \delta) : (u, v) \in E(G_i)\}$. 
We also make the edge features of these events match those in the DTDG, i.e., $e_{u v}(i \delta) = e_{u v} \in \mathcal{E}_i$.
To account for node features, for all $u \in V(G_i)$, we create an event $(u, \diamond, i \delta)$ between $u$ and a dummy node $\diamond$, with feature $e_{u \diamond}(i \delta)=x_u \in \mathcal{X}_i$.
Let $C(G_i)$ denote the set comprising these node-level events.
Then, we can construct the CTDG $\mathsf{G}(t_i) = \cup_{j=1}^i S(G_j) \cup C(G_j)$ for $i=1, \dots$.
Reconstructing the DTDG $(G_1, G_2, \dots)$ is trivial. To build $G_i$, it suffices to select all events at time $i \delta$ in the CTDG. Events involving $\diamond$ determine node features and the remaining ones constitute edges in the DTDG.

\textbf{\textcolor{lb}{Statement 2:}} The converse holds if the CTDG timestamps form a subset of some uniformly spaced countable set.

We say that a countable set $A \subset \mathbb{R}$ is uniformly spaced if there exists some $\delta \in \mathbb{R}$ such that $a_{i+1} - a_i = \delta$ for all $i$ where $(a_1, a_2, \ldots)$ is the ordered sequence formed from elements $a_r$ of $A$, i.e., $a_1 < a_2 < \ldots < a_i < a_{i+1}, \ldots$

Note that DTDGs are naturally represented by a set of uniformly spaced timestamps. This is because DTDGs correspond to sequences that do not contain any time information.
Let us denote the set of CTDG timestamps $T \subseteq \mathcal{T}$ such that $\mathcal{T}$ is countable and uniformly spaced.
Our idea is to construct a DTDG sequence with timestamps that coincide with the elements in $\mathcal{T}$. Then, since $T \subseteq \mathcal{T}$, we do not lose any information pertaining to events occurring at timestamps given by $T$.
Without loss of generality, in the following we assume that the elements of $T$ and $\mathcal{T}$ are arranged in their increasing order respectively, i.e., $t_i < t_{i+1}$ for all $i$, and $\tau_k < \tau_{k+1}$ for all $k$.

Consider a CTDG $(\mathsf{G}(t_1), \mathsf{G}(t_2), \dots)$ such that $\mathsf{G}(t_i) = \{(u, v, t): t \in T \text{ and } t \leq t_i\}$ for $t_i \in T$.
Also, let us denote $H(t_i)=\{(u, v, t) \in \mathsf{G}(t_i): t=t_i\}$ the set of events at time $t_i \in T$.
We can build a corresponding DTDG $(G_{1}, G_{2}, \dots)$ such that for all $\tau_k \in \mathcal{T}$ the $k$-th snapshot $G_k$ is
\begin{align*}
    & V(G_k) = \begin{cases}
                \{u: (u, \cdot, \tau_k) \in H(\tau_k)\}, & \text{ if } \tau_k \in T;\\
                \emptyset, & \text{otherwise}.
             \end{cases}    
             \\ 
   & E(G_k) = \begin{cases}
             \{(u, v): (u, v, \tau_k) \in H(\tau_k)\}, & \text{if } \tau_k \in T; \\
              \emptyset, &  \text{otherwise}.
             \end{cases}        
\end{align*}

To recover the original CTDG, we can adapt the reconstruction procedure we used in the previous part of the proof.  We define 
\begin{align}
   \tilde{I} = \{(i, k) \in \mathbb{N} \times \mathbb{N} : \tau_k = t_i \text{ for } t_i \in T \text{ and } \tau_k \in \mathcal{T}\}.
\end{align}
Note that we can treat $\tilde{I}$ as a map by defining $\tilde{I}(i)= k$ if and only if $(i, k) \in \tilde{I}$. 
To recover the original CTDG, we first create the set of events $S(G_k) = \{(u, v, k \delta) : (u, v) \in E(G_k)\}$.
Then, we build $\mathsf{G}(t_i) = \cup_{j: j \leq \tilde{I}(i)} S(G_j)$ for $t_i \in T$.
\end{proof}

\subsection{Proof of \autoref{lemma:1}}
\begin{proof} 
Here we show that if two nodes $u$ and $v$ have isomorphic ($L$-depth) TCTs, then \mptgns (with $L$-layers) compute identical embeddings for $u$ and $v$. Formally, let $T_{u,\ell}(t)$ denote the TCT of $u$ with $\ell$ layers. We want to show that $T_{u,\ell}(t) \cong T_{v, \ell}(t) \Rightarrow h_u^{(\ell)}(t) = h_v^{(\ell)}(t)$. We employ a proof by induction on $\ell$. Since there is no ambiguity, we drop the dependence on time in the following.

[\emph{Base case}] Consider $\ell=1$. By the isomorphism assumption $T_{u, 1} \cong T_{v, 1}$, $h_u^{(0)}=s_u=s_v=h_v^{(0)}$ --- roots of both trees have the same states. Also, for any children $i$ of $u$ in $T_{u, 1}$ there is a corresponding one $f(i)$ in $T_{v,1}$ with $s_i=s_{f(i)}$. Recall that the $\ell$-th layer aggregation function $\textsc{Agg}^{(\ell)}(\cdot)$ acts on multisets of triplets of previous-layer embeddings, edge features and timestamps of temporal neighbors (see \autoref{eq:aggregation}).
Since the temporal neighbors of $u$ correspond to its children in $T_{u, 1}$, then the output of the aggregation function for $u$ and $v$ are identical: $\tilde{h}_u^{(1)}=\tilde{h}_v^{(1)}$. In addition, since the initial embeddings of $u$ and $v$ are also equal (i.e., $h_u^{(0)}=h_v^{(0)}$), we can ensure that the update function returns $h_u^{(1)}=h_v^{(1)}$.

[\emph{Induction step}] Assuming that $T_{u,\ell-1} \cong T_{v, \ell-1} \Rightarrow h_u^{(\ell-1)}= h_v^{(\ell-1)}$ for any pair of nodes $u$ and $v$, we will show that $T_{u,\ell} \cong T_{v, \ell} \Rightarrow h_u^{(\ell)} = h_v^{(\ell)}$. For any children $i$ of $u$, let us define the subtree of $T_{u,\ell}$ rooted at $i$ by $T_{i}$. We know that $T_{i}$ has depth $\ell-1$, and since $T_{u,\ell} \cong T_{v, \ell}$, there exists a corresponding subtree of $T_v$ (of depth $\ell-1$) rooted at $f(i)$ such that $T_{i} \cong T_{f(i)}$. Using the induction hypothesis, we obtain that the multisets of embeddings from the children $i$ of $u$ and children $f(i)$ of $v$ are identical. Note that if two $\ell$-depth TCTs are isomorphic, they are also isomorphic up to depth $\ell-1$, i.e., $T_{u,\ell} \cong T_{v,\ell}$ implies $T_{u,\ell-1} \cong T_{v,\ell-1}$ and, consequently, $h^{(\ell-1)}_u=h^{(\ell-1)}_{v}$ (by induction hypothesis). Thus, the input of the aggregation and update functions are identical and they compute the same embeddings for $u$ and $v$.
\end{proof}

\subsection{Proof of \autoref{prop:injectiveness-requirement}: Most expressive \mptgns}
\begin{proof} 
Consider \mptgns with parameter values that make $\textsc{Agg}^{(\ell)}(\cdot)$ and $\textsc{Update}^{(\ell)}(\cdot)$ injective functions on multisets of triples of hidden representations, edge features and timestamps. The existence of these parameters is guaranteed by the fact that, at any given time $t$, the space of node states (and hidden embeddings from temporal neighbors), edge features and timestamps is finite (see \autoref{lemma:lemma5_gin}). 

Again, let $T_{u,\ell}(t)$ denote the TCT of $u$ with $\ell$ layers. We want to prove that, under the injectivity assumption, if $T_{u, \ell}(t) \not\cong T_{v, \ell}(t)$, then $h^{(\ell)}_u(t) \neq h^{(\ell)}_u(t)$ for any two nodes $u$ and $v$. In the following, we simplify notation by removing the dependence on time. We proceed with proof by induction on the TCT's depth $\ell$. Also, keep in mind that $\varphi_\ell = \textsc{Update}^{(\ell)} \circ \textsc{Agg}^{(\ell)}$ is injective for any $\ell$.

[\emph{Base case}] For $\ell=1$, if $T_{u, 1} \not\cong T_{v, 1}$ then the root node states are different (i.e., $s_u\neq s_v$) or the multiset of states/edge features/ timestamps triples from $u$ and $v$'s children are different. In both cases, the inputs of $\varphi_{\ell}$ are different and it therefore outputs different embeddings for $u$ and $v$.

[\emph{Induction step}] The inductive hypothesis is $T_{u, \ell-1} \not\cong T_{v, \ell-1} \Rightarrow h^{(\ell-1)}_u \neq h^{(\ell-1)}_v$ for any pair of nodes $u$ and $v$. 
If $T_{u, \ell} \not\cong T_{v, \ell}$, at least one of the following holds: i) the states of $u$ and $v$ are different, ii) the multisets of edges (edge features/ timestamps) with endpoints in $u$ and $v$ are different, or iii) there is no pair-wise isomorphism between the TCTs rooted at $u$ and $v$'s children. In the first two cases, $\varphi_{\ell}$ trivially outputs different embeddings for $u$ and $v$. We are left with the case in which only the latter occurs. Using our inductive hypothesis, the lack of a (isomorphism ensuring) bijection between the TCTs rooted at $u$ and $v$'s children implies there is also no bijection between their multiset of embeddings. In turn, this guarantees that $\varphi$ will output different embeddings for $u$ and $v$.
\end{proof}

\subsection{Proof of \autoref{prop:memory_expressivity}: The role of memory}
\begin{proof}
We prove the two parts of the proposition separately. In the following proofs, we rely on the concept of monotone TCTs (see \autoref{def:monotone_tct}).

\textbf{\textcolor{lb}{Statement 1:}} If $L < \Delta$: $\mathcal{Q}_L^{[M]}$ is strictly stronger than $\mathcal{Q}_L$.

We know that the family of $L$-layer \mptgns with memory comprises the family of $L$-layer \mptgns without memory (we can assume identity memory). Therefore, $\mathcal{Q}_L^{[M]}$ is at least as powerful as  $\mathcal{Q}_L$. To show that $\mathcal{Q}_L^{[M]}$ is strictly stronger (more powerful) than $\mathcal{Q}_L$, when $L < \Delta$, it suffices to create an example for which memory can help distinguish a pair of nodes. We provide a trivial example in \autoref{fig:memory-vs-identity} for $L=1$. Note that the $1$-depth TCTs of $u$ and $v$ are isomorphic when no memory is used. However, when equipped with memory, the interaction $(b, c, t_1)$ affects the states of $v$ and $c$, making the 1-depth TCTs of $u$ and $v$ (at time $t>t_2$) no longer isomorphic. 

\begin{figure}[ht]
    \centering
    \includegraphics[width=0.25\textwidth]{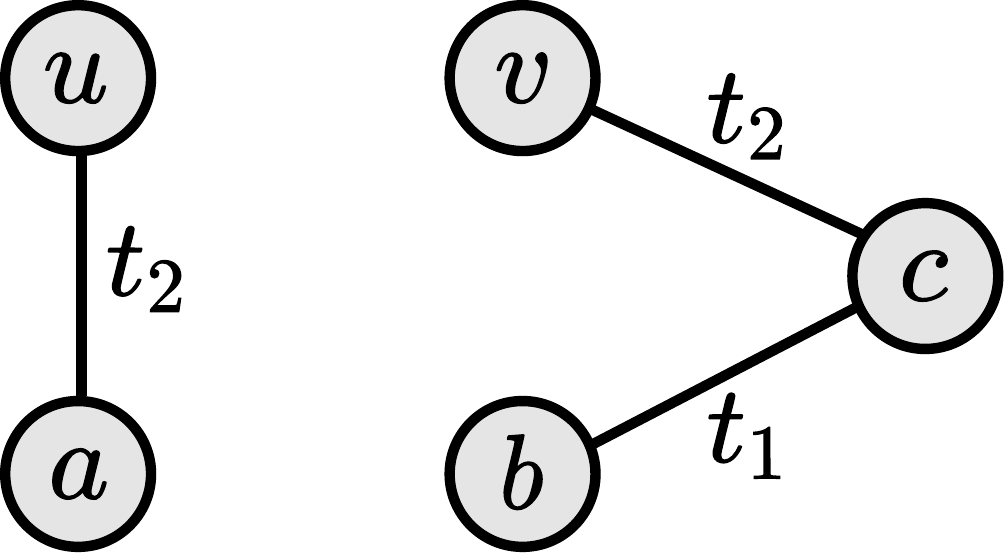}
    \caption{Temporal graph where all initial node features and edge features are identical, and $t_2 > t_1$.}
    \label{fig:memory-vs-identity}
\end{figure}

\textbf{\textcolor{lb}{Statement 2:}} For any $L$ : $\mathcal{Q}_{L+\Delta}$ is at least as powerful as $\mathcal{Q}^{[M]}_{L}$.

It suffices to show that if $\mathcal{Q}_{L+\Delta}$ cannot distinguish a pair of nodes $u$ and $v$ , $\mathcal{Q}^{[M]}_{L}$ cannot distinguish them too. Let $T_{u,L}^{M}(t)$ and $T_{u,L}(t)$ denote the $L$-depth TCTs of $u$ with and without memory respectively. Using \autoref{lemma:1}, this is equivalent to showing that $T_{u, L+\Delta}(t) \cong T_{v, L+\Delta}(t) \Rightarrow T_{u, L}^{M}(t) \cong T_{v, L}^{M}(t)$, since no \mptgn can separate nodes associated with isomorphic TCTs. In the following, when we omit the number of layers from TCTs, we assume TCTs of arbitrary depth.

\textbf{\textcolor{Mahogany}{Step 1:}} 
Characterizing the dependence of memory on initial states and events in the dynamic graph.

We now show that the memory for a node $u$, after processing all events with timestamp $\leq t_n$, depends on the initial states of a set of nodes $\mathcal{V}_{u}^n$, and a set of events annotated with their respective timestamps and features $\mathcal{B}_{u}^n$. If at time $t_n$ no event involves a node $z$ , we set  $\mathcal{B}_{z}^n = \mathcal{B}_{z}^{n-1}$ and $\mathcal{V}_{z}^n = \mathcal{V}_{z}^{n-1}$. 
We also initialize $\mathcal{B}_{u}^0 = \emptyset$ and $\mathcal{V}_{u}^0 = \{u\}$ for all nodes $u$. We proceed with a proof by induction on the number of observed timestamps $n$.

[\emph{Base case}] 
Let $\mathcal{I}_{1}(u)=\{v :(u, v, t_1) \in \mathcal{G}(t_1^+)\}$ be the set of nodes interacting with $u$ at time $t_1$, where $\mathcal{G}(t_1^+)$ is the temporal graph right after $t_1$. Similarly, let $\mathcal{J}_{1}(u)= \{(u, \cdot, t_1) \in \mathcal{G}(t_1^+)\}$ be the set of events involving $u$ at time $t_1$. Recall that up until $t_1$, all memory states equal initial node features (i.e., $s_{u}(t_{1}) = s_{u}(0)$). 
Then, the updated memory (see \autoref{eq:mem-msg} and \autoref{eq:memory}) for $u$ depends on $\mathcal{V}_{u}^1 = \mathcal{V}_{u}^0 \cup_{v \in \mathcal{I}(u)} \mathcal{V}_{v}^0$, $\mathcal{B}_{u}^1 = \mathcal{B}_{u}^0 \cup \mathcal{J}_{1}(u)$.

[\emph{Induction step}] Assume that for timestamp $t_{n-1}$ the proposition holds. We now show that it holds for $t_n$. 
Since the proposition holds for $n - 1$ timestamps, we know that the memory of any $w$ that interacts with $u$ in $t_n$, i.e. $w \in \mathcal{I}_{n}(u)$, depends on $\mathcal{V}_{w}^{n-1}$ and $\mathcal{B}_{w}^{n-1}$, and the memory of $u$ so far depends on $\mathcal{V}_{u}^{n-1}$ and $\mathcal{B}_{u}^{n-1}$.
Then, the updated memory for $u$ depends on $\mathcal{V}_{u}^{n} =  \mathcal{V}_{u}^{n-1} \cup_{w \in \mathcal{I}(u)} \{ w, u\}  \cup \mathcal{V}_{w}^{n-1}$ and
$\mathcal{B}_{u}^{n} = \mathcal{B}_{u}^{n-1} \cup \mathcal{J}_{n}(u) \cup_{w \in \mathcal{I}(u)} \mathcal{B}_{w}^{n-1} $.

\textbf{\textcolor{Mahogany}{Step 2:}} $(z, w, t_{zw}) \in \mathcal{B}^n_u$ if and only if there is a path $(u_k, t_k=t_{zw}, u_{k+1})$ in $\tilde{T}_u(t_n^+)$ --- the monotone TCT of $u$ (see \autoref{def:monotone_tct}) after processing events with timestamp $\leq t_n$ --- with either $\sharp u_{k}=z, \sharp u_{k+1}=w$  or $\sharp u_{k}=w, \sharp u_{k+1}=z$.

[\emph{Forward direction}]
An event $(z, w, t_{zw})$ with $t_{zw} \leq t_n$ will be in $\mathcal{B}^n_u$ only if $z=u$ or $w=u$, or if 
there is a subset of events $\{(u, \sharp u_1, t_1), (\sharp u_1, \sharp u_2, t_2), \dots, (\sharp u_k, \sharp u_{k+1}, t_{zw})\}$ with $\sharp u_k = z$ and $\sharp u_{k+1} = w$ such that $t_n \geq t_{1} > \dots > t_{zw}$. 
In either case, this will lead to root-to-leaf path in $\tilde{T}_u(t_n^{+})$ passing through $(u_k, t_{zw}, u_{k+1})$.
This subset of events can be easily obtained by backtracking edges that caused unions/updates in the procedure from {\color{Mahogany} Step 1}.

[\emph{Backward direction}] Assume there is a subpath $p=(u_k, t_k=t_{zw}, u_{k+1}) \in \tilde{T}_u(t_n^{+})$ 
with $\sharp u_k = z$ and $\sharp u_{k+1} = w$ such that $(z, w, t_{zw}) \notin \mathcal{B}_u^n$. Since we can obtain $p$ from $\tilde{T}_u(t_n^{+})$, we know that the sequence of events $r = ((u, \sharp u_1, t_1), \dots, (\sharp u_{k-2}, \sharp u_{k-1}=z, t_{k-1}),  (z, w, t_k=t_{zw}))$ happened and that $t_{i} > t_{i+1} \quad \forall i$. However, since $(z, w, t_{zw}) \notin \mathcal{B}_u^n$, there must be no monotone walk starting from $u$ going through the edge $(z, w, t_{zw})$ to arrive at $w$, which is exactly what $r$ characterizes. Thus, we reach contradiction. 

Note that the nodes in $\mathcal{V}^n_u$ are simply the nodes that have an endpoint in the events $\mathcal{B}^n_u$, and therefore are also nodes in $\tilde{T}_u(t_n^+)$ and vice-versa.

\textbf{\textcolor{Mahogany}{Step 3:}} For any node $u$, there is a bijection  that maps ($\mathcal{V}^n_u, \mathcal{B}^n_u$) to $\tilde{T}_u(t_n^+)$.

First, we note that ($\mathcal{V}^n_u, \mathcal{B}^n_u$) depends on a subset of all events, which we represent as $\mathcal{G}^\prime \subseteq \mathcal{G}(t_n^+)$. Since $\mathcal{B}^n_{u}$ contains all events in $\mathcal{G}^\prime$ and $(\mathcal{V}^n_{u}, \mathcal{B}^n_{u})$ can be uniquely constructed from $\mathcal{G}^\prime$, then there is a bijection $g$ that maps from $\mathcal{G}^\prime$ to $(\mathcal{V}^n_{u}, \mathcal{B}^n_{u})$.

Similarly, $\tilde{T}_{u}(t_n^+)$ also depends on a subset of events which we denote by $\mathcal{G}^{\prime\prime} \subseteq \mathcal{G}(t_n^+)$. We note that the unique events in $\tilde{T}_{u}(t_n^+)$ correspond to $\mathcal{G}^{\prime\prime}$, and we can uniquely build the tree $\tilde{T}_{u}(t_n^+)$ from $\mathcal{G}^{\prime\prime}$. This implies that there is a bijection $h$ that maps from $\mathcal{G}^{\prime\prime}$ to $\tilde{T}_{u}(t_n^+)$.

Previously, we have shown that all events in $\mathcal{B}^n_{u}$ are also in $\tilde{T}_{u}(t_n^+)$ and vice-versa. This implies that both sets depend on the same events, and thus on the same subset of all events, i.e., $\mathcal{G}^{\prime} = \mathcal{G}^{\prime\prime} = \mathcal{G}_{S}$. Since there is a bijection $g$ between $\mathcal{G}_{S}$  and $(\mathcal{V}^n_{u}, \mathcal{B}^n_{u})$, and a bijection $h$ between $\mathcal{G}_{S}$ and $\tilde{T}_{u}(t_n^+)$, there exists a bijection $f$ between $(\mathcal{V}^n_{u}, \mathcal{B}^n_{u})$ and $\tilde{T}_{u}(t_n^+)$.

\textbf{\textcolor{Mahogany}{Step 4:}} If $T_{u, L+\Delta}(t^+) \cong T_{v, L+\Delta}(t^+)$, then $T_{u, L}^{M}(t^+) \cong T_{v, L}^{M}(t^+)$.

To simplify notation, we omit here the dependence on time. 

Any node $w \in T_{u, L}^{M}$ also appears in $T_{u, L+\Delta}$ at the same level. The subtree of $T_{u, L+\Delta}$ rooted at $w$, denoted here by $T_w^\prime$, has depth at least $k \geq \Delta$. 
Note that $T_w^\prime$ corresponds to the $k$-depth TCT of $\sharp w$. 
Since the depth of $T_w^\prime$ is at least $\Delta$, we know that the $\tilde{T}_{w}^\prime \cong \tilde{T}_{\sharp w}$ --- i.e., imposing time-constraints to $T^\prime_{w}$ results in the {\tcTCT} of node $\sharp w$. 
Also, because the memory of $\sharp w$ depends on $\tilde{T}_{\sharp w}$, $T_w^\prime$ comprises the information used to compute the memory state of $\sharp w$. Note that this applies to any $w$ in $T_{u, L}^{M}$; thus, $T_{u, L+\Delta}$ contains all we need to compute the states of any node of the dynamic graph that appears in $T_{u, L}^{M}$. The same argument applies to $T_{v, L+\Delta}$ and $T_{v, L}^{M}$.
Finally, since $T_{u, L}^{M}$ can be uniquely computed from $T_{u, L+\Delta}$, and $T_{v, L}^{M}$ from $T_{v, L+\Delta}$, if $T_{u, L+\Delta} \cong T_{v, L+\Delta}$, then $T_{u, L}^{M} \cong T_{v, L}^{M}$.
\end{proof}

\subsection{Proof of \autoref{prop:TGAT-TGN-Att}: Limitations of TGAT and TGN-Att}\label{ap:limitations_TGAT}
\begin{proof} 
In this proof, we first provide an example of a dynamic graph where the TCTs of two nodes $u$ and $v$ are not isomorphic. Then, we show that we can not find a TGAT model such that $h_u^{(L)}(t) \neq h_v^{(L)}(t)$, i.e., TGAT does not distinguish $u$ and $v$. Next, we show that even if we consider TGATs with memory (TGN-Att), it is still not possible to distinguish nodes $u$ and $v$ in our example. 

\begin{figure}[htb]
    \centering
    \includegraphics[width=\textwidth]{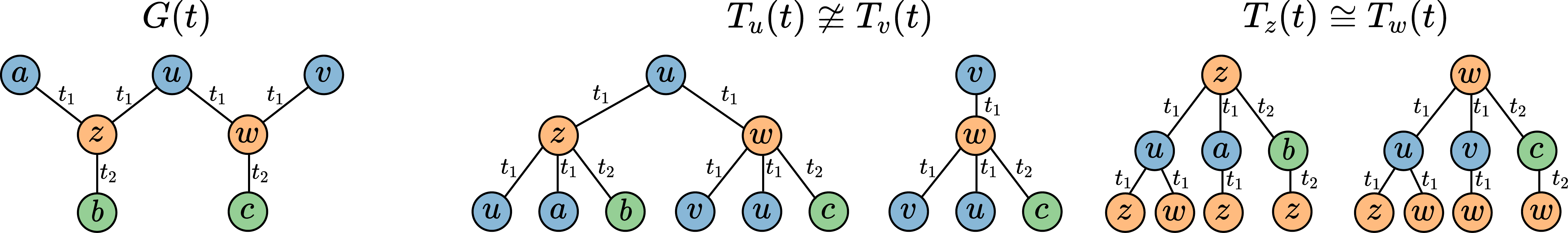}
    \caption{(\textbf{Leftmost}) Example of a temporal graph for which TGN-Att and TGAT cannot distinguish nodes $u$ and $v$ even though their TCTs are non-isomorphic. Colors denote node features and all edge features are identical, and $t_2 > t_1$ (and $t > t_2$). (\textbf{Right}) The 2-depth TCTs of nodes $u, v, z$ and $w$. The TCTs of $u$ and $v$ are non-isomorphic whereas the TCTs of $z$ and $w$ are isomorphic.}
    \label{fig:tgn_fail}
\end{figure}

\autoref{fig:tgn_fail}(leftmost) provides a temporal graph where all edge events have the same edge features. Colors denote node features. As we can observe, the TCTs of nodes $u$ and $v$ are not isomorphic. In the following, we consider node distinguishability at time $t > t_2$.

\textbf{\textcolor{lb}{Statement 1:}} TGAT cannot distinguish the nodes $u$ and $v$ in our example.

\textbf{\textcolor{Mahogany}{Step 1:}} For any TGAT with $\ell$ layers, we have that $h_w^{(\ell)}(t)=h_z^{(\ell)}(t)$.

We note that the $\ell$-layer TCTs of nodes $w$ and $z$ are isomorphic, for any $\ell$. To see this, one can consider the symmetry around node $u$ that allows us to define a node permutation function (bijection) $f$ given by $f(z)=w, f(w)=z, f(a)=v, f(u)=u, f(b)=c, f(c)=b, f(v)=a$. \autoref{fig:tgn_fail}(right) provides an illustration of the $2$-depth TCTs of $z$ and $w$ at time $t>t_2$.

By Lemma 1, if the $\ell$-layer TCTs of two nodes $z$ and $w$ are isomorphic, then no $\ell$-layer \mptgn can distinguish them. Thus, we conclude that $h_w^{(\ell)}(t)=h_z^{(\ell)}(t)$ for any TGAT with arbitrary number of layers $\ell$.

\textbf{\textcolor{Mahogany}{Step 2:}} There is no TGAT such that $h_v^{(\ell)}(t) \neq h_u^{(\ell)}(t)$.

To compute $h_v^{(\ell)}(t)$, TGAT aggregates the messages of $v$'s temporal neighbors at layer $\ell-1$, and then combines $h_v^{(\ell-1)}(t)$ with the aggregated message $\tilde{h}_v^{(\ell-1)}(t)$ to obtain $h_v^{(\ell)}(t)$. 

Note that $\mathcal{N}(u, t)=\{(z, e, t_1), (w, e, t_1)\}$ and $\mathcal{N}(v, t)=\{(w, e, t_1)\}$, where $e$ denotes an edge feature vector. Also, we have previously shown that $h_w^{(\ell-1)}(t)=h_z^{(\ell-1)}(t)$. 

Using the TGAT aggregation layer (\autoref{eq:tgat-aggregation}), the query vectors of $u$ and $v$ are $q_u = [h_u^{(\ell-1)}(t) || \phi(0)]W^{(\ell)}_q$ and $q_v = [h_v^{(\ell-1)}(t) || \phi(0)]W^{(\ell)}_q$, respectively.

Since all events have the common edge features $e$, the matrices $C_u^{(\ell)}$ and $C_v^{(\ell)}$ share the same vector in their rows. The single-row matrix $C_v^{(\ell)}$ is given by $C_v^{(\ell)} = [h_w^{(\ell-1)}(t) || \phi(t-t_1) || e]$, while the two-row matrix $C^{(\ell)}_u = \left[[h_w^{(\ell-1)}(t) || \phi(t-t_1) || e]; [h_z^{(\ell-1)}(t) \| \phi(t-t_1) || e]\right]$, with $h_w^{(\ell-1)}(t) = h_z^{(\ell-1)}(t)$. We can express $C_u^{(\ell)} = [1, 1]^\top r$ and $C_v^{(\ell)} = r$ , where $r$ denotes the row vector $r=[h_z^{(\ell-1)}(t) || \phi(t-t_1) || e]$.

Using the key and value matrices of node $v$, i.e.,
$K_v=C_v^{(\ell)} W^{(\ell)}_K$ and  
$V_v=C_v^{(\ell)} W^{(\ell)}_V$, we have that
\begin{align}
\tilde{h}^{(\ell)}_v(t) & = \mathrm{softmax}(q_v K_v^\top) V_v & \nonumber \\       
& = \underbrace{\mathrm{softmax}(q_v K_v^\top)}_{=1} r W^{(\ell)}_V  & \text{ [softmax of a single element is 1] } \nonumber\\
& = r W^{(\ell)}_V \nonumber \\
& = \underbrace{\mathrm{softmax}(q_u K_u^\top)[1, 1]^\top}_{=1} r W^{(\ell)}_V  = \tilde{h}^{(\ell)}_u(t) & \text{[softmax outputs a convex combination]} \nonumber
\end{align}

We have shown that the aggregated messages of nodes $u$ and $v$ are the same at any layer $\ell$. We note that the initial embeddings are also identical $h^{(0)}_v(t) = h^{(0)}_u(t)$ as $u$ and $v$ have the same color. Recall that the update step is $h_v^{(\ell)}(t)=\mathrm{MLP}(h_v^{(\ell-1)}(t), \tilde{h}_v^{(\ell)}(t))$. Therefore, if the initial embeddings are identical, and the aggregated messages at each layer are also identical, we have that $h_u^{(\ell)}(t) = h_v^{(\ell)}(t)$ for any $\ell$. 

\textbf{\textcolor{lb}{Statement 2:}} TGN-Att cannot distinguish the nodes $u$ and $v$ in our example.

We now show that adding a memory module to TGAT produces node states such that $s_u(t)=s_v(t)=s_a(t)$, $s_z(t)=s_w(t)$, and $s_b(t)=s_c(t)$. If that is the case, then these node states could be treated as node features in a equivalent TGAT model of our example in \autoref{fig:tgn_fail}, proving that there is no TGN-Att such that $h_v^{(\ell)}(t) \neq h_u^{(\ell)}(t)$. In the following, we consider TGN-Att with average memory aggregators (see Appendix \ref{ap:egn}).

We begin by showing that $s_{a}(t) = s_{u}(t) = s_{v}(t)$ after memory updates. We note that the memory message node $a$ receives is $[e \| t_{1} \| s_{z}(t_{1})]$. The memory message node $u$ receives is $\textsc{MeanAgg}([e \| t_{1} \| s_{w}(t_{1})], [e \| t_{1} \| s_{z}(t_{1})])$, but since $s_{w}(t_{1}) = s_{z}(t_{1})$, both messages are the same, and the average aggregator outputs $[e \| t_{1} \| s_{z}(t_{1})]$. Finally, the message that node $v$ receives is $[e \| t_{1} \| s_{w}(t_{1})] = [e \| t_{1} \| s_{z}(t_{1})]$. Since all three nodes receive the same memory message and have the same initial features, their updated memory states are identical.

Now we show that $s_{z}(t) = s_{w}(t)$, for $t_{1} < t \leq t_{2}$. Note that the message that node $z$ receives is $\textsc{MeanAgg}([e \| t_{1} \| s_{a}(t_{1})], [e \| t_{1} \| s_{u}(t_{1})]) = [e \| t_{1} \| s_{u}(t_{1})]$, with $s_{u}(t_{1}) = s_{a}(t_{1})$. The message that node $w$ receives is $\textsc{MeanAgg}([e \| t_{1} \| s_{u}(t_{1})], [e \| t_{1} \| s_{v}(t_{1})]) = [e \| t_{1} \| s_{u}(t_{1})]$. Again, since the initial features and the messages received by each node are equal, $s_{z}(t) = s_{w}(t)$ for $t_{1} < t \leq t_{2}$.

We can then use this to show that $s_{z}(t) = s_{w}(t)$ for $t > t_{2}$. Note that at time $t_{2}$, the message that nodes $z$ and $w$ receive are $[e \| t_{2} - t_1 \| s_{b}(t_{2})]$ and $[e \| t_{2} - t_{1} \| s_{c}(t_{2})]$, respectively. Also, note that $s_b(t_2)=s_c(t_2)=s_b(0)=s_c(0)$ as the states of $b$ and $c$ are only updated right after $t_2$. Because the received messages and the previous states (up until $t_2$) of $z$ and $w$ are identical, we have that $s_{z}(t) = s_{w}(t)$ for $t > t_{2}$.

Finally, we show that $s_{b}(t) = s_{c}(t)$. Using that $s_{z}(t_{2}) = s_{w}(t_{2})$ in conjunction with the fact that node $b$ receives message $[[e \| t_{2} - t_{1} \| s_{z}(t_{2})]]$, and node $c$ receives $[e \| t_{2} - t_{1} \| s_{w}(t_{2})]$, we obtain $s_{b}(t) = s_{c}(t)$ since initial memory states and messages that the nodes received are the same.
\end{proof}

\subsection{Proof of \autoref{prop:tgns_vs_caw}: Limitations of \mptgn and CAWs}
\begin{figure}[ht] 
    \centering
    \includegraphics[width=0.9\textwidth]{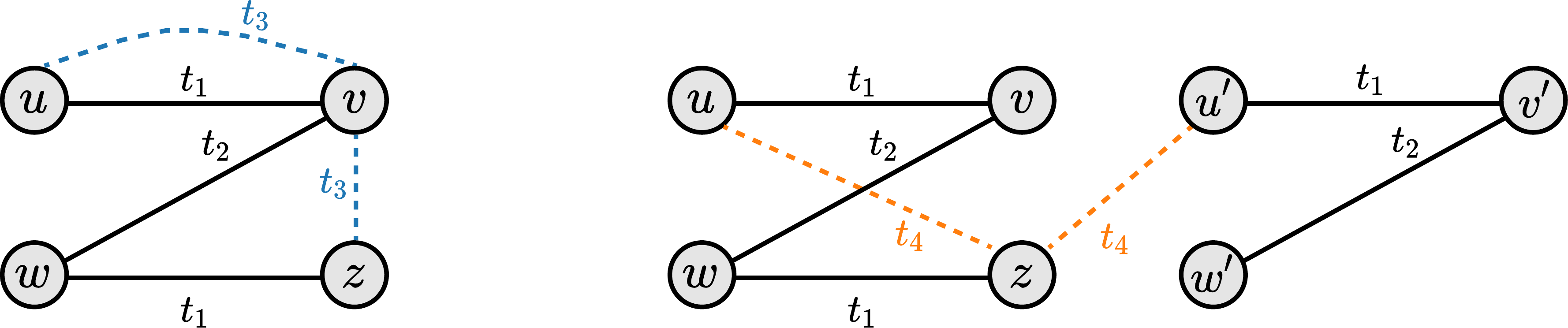}
    \caption{(\textbf{Left}) Example of a temporal graph for which CAW can distinguish the events $\color{lb}(u, v, t_3)\color{black}$ and $\color{lb}(z, v, t_3)\color{black}$ but \mptgns cannot. We assume that all edge and node features are identical, and $t_{k+1}>t_{k}$ for all $k$. (\textbf{Right}) Example for which \mptgns can distinguish $\color{orange}(u, z, t_4)\color{black}$ and $\color{orange}(u^\prime, z, t_4)\color{black}$ but CAW cannot.}
    \label{fig:caw_vs_tgn}
\end{figure}

\begin{proof}
Using the example in \autoref{fig:caw_vs_tgn}(Left), we adapt a construction by \citet{caw} to show that CAW can separate events that \mptgns adopting node embedding concatenation cannot. We first note that the TCTs of $u$ and $z$ are isomorphic. Thus, since $v$ is a common endpoint in $(u, v, t_3)$ and $(z, v, t_3)$, no \mptgn can distinguish these two events. Nonetheless, CAW obtains the following anonymized walks for the event $(u, v, t_3)$:
\begin{align*}
&\underbrace{\{[1,0,0], [0, 1, 0]\}}_{I_{\text{CAW}}(u; S_u, S_v)} \xrightarrow{t_1} \underbrace{\{[0,1,0], [2, 0, 0]\}}_{I_{\text{CAW}}(v; S_u, S_v)}   \\ 
& \underbrace{\{[0,1,0], [2, 0, 0]\}}_{I_{\text{CAW}}(v; S_u, S_v)} \xrightarrow{t_1} \underbrace{\{[1,0,0], [0, 1, 0]\}}_{I_{\text{CAW}}(u; S_u, S_v)}   \\ 
& \underbrace{\{[0,1,0], [2, 0, 0]\}}_{I_{\text{CAW}}(v; S_u, S_v)} \xrightarrow{t_2} \underbrace{\{[0,0,0], [0, 1, 0]\}}_{I_{\text{CAW}}(w; S_u, S_v)} \xrightarrow{t_1} \underbrace{\{[0,0,0], [0, 0, 1]\}}_{I_{\text{CAW}}(z; S_u, S_v)} 
\end{align*}
and the walks associated with $(z, v, t_3)$ are (here we omit underbraces for readability):
\begin{align*}
&\{[1,0,0], [0, 0, 1]\} \xrightarrow{t_1} \{[0,1,0], [0, 1, 0]\} \\
&\{[0,0,0], [2, 0, 0]\} \xrightarrow{t_1} \{[0,0,0], [0, 1, 0]\}   \\ 
&\{[0,0,0], [2, 0, 0]\} \xrightarrow{t_2} \{[0,1,0], [0, 1, 0]\} \xrightarrow{t_1} \{[1,0,0], [0, 0, 1]\}  
\end{align*}

In this example, assume that MLPs used to encode each walk correspond to identity mappings. Then, the sum of the elements in each set is injective since each element of the sets in the anonymized walks are one-hot vectors. We note that, in this example, we can simply choose a RNN that sums the vectors in each sequence (walks), and then apply a mean readout layer (or pooling aggregator) to obtain distinct representations for $(u, v, t_3)$ and $(z, v, t_3)$. 

We now use the example in \autoref{fig:caw_vs_tgn}(Right) to show that \mptgns can separate events that CAW cannot. To see why \mptgns can separate the events $(u, z, t_4)$ and $(u^\prime, z, t_4)$, it suffices to observe that the 4-depth TCTs of $u$ and $u^\prime$ are non-isomorphic. Thus, a \mptgn with injective layers could distinguish such events. Now, let us take a look at the anonymized walks for $(u, z, t_4)$:
\begin{align*}
& \underbrace{\{[1,0], [0, 0]\}}_{I_{\text{CAW}}(u; S_u, S_z)} \xrightarrow{t_1} \underbrace{\{[0,1], [0, 0]\}}_{I_{\text{CAW}}(v; S_u, S_z)}   \\ 
&\underbrace{\{[0, 0], [1, 0]\}}_{I_{\text{CAW}}(z; S_u, S_z)} \xrightarrow{t_1} \underbrace{\{[0,0], [0, 1]\}}_{I_{\text{CAW}}(w; S_u, S_z)}   
\end{align*}
and for $(u^\prime, z, t_4)$:
\begin{align*}
& \underbrace{\{[1,0], [0, 0]\}}_{I_{\text{CAW}}(u^\prime; S_{u^\prime}, S_z)} \xrightarrow{t_1} \underbrace{\{[0,1], [0, 0]\}}_{I_{\text{CAW}}(v^\prime; S_{u^\prime}, S_z)}   \\ 
& \underbrace{\{[0, 0], [1, 0]\}}_{I_{\text{CAW}}(z; S_{u^\prime}, S_z)}  \xrightarrow{t_1} \underbrace{\{[0,0], [0, 1]\}}_{I_{\text{CAW}}(w; S_{u^\prime}, S_z)}   
\end{align*}

Since the sets of walks are identical, they must have the same embedding. Therefore, there is no CAW model that can separate these two events.
\end{proof}

\subsection{Proof of \autoref{cor:temporal-wl-tgns}: Injective \mptgns and the temporal WL test}

We want to prove that injective \mptgns can separate two temporal graphs if and only if the temporal WL does the same. Our proof comprises two parts. We first show that if an \mptgn produces different multisets of embeddings for two non-isomorphic temporal graphs $\mathcal{G}(t)$ and $\mathcal{G}^\prime(t)$, then the temporal WL decides these graphs are not isomorphic. Then, we prove that, if the temporal WL decides $\mathcal{G}(t)$ and $\mathcal{G}^\prime(t)$ are non-isomorphic, there is an injective \mptgn (i.e., with injective message-passing layers) that outputs distinct multisets of embeddings.

\textbf{\textcolor{lb}{Statement 1:}} Temporal WL is at least as powerful as \mptgns.

See \autoref{lem:b2} for proof.

\textbf{\textcolor{lb}{Statement 2:}} Injective \mptgn is at least as powerful as temporal WL.
\begin{proof}
To prove this, we can repurpose the proof of Theorem 3 in \citep{gin}. In particular, we assume \mptgns that meet the injective requirements of Proposition 2, i.e., \mptgns that implement injective aggregate and update functions on multisets of hidden representations from temporal neighbors. Following their footprints, we prove that there is a injection $\varphi$ to the set of embeddings of all nodes in a temporal graph from their respective colors in the temporal WL test. We do so via induction on the number of layers $\ell$. To achieve our purpose, we can assume identity memory without loss of generality.

The base case ($\ell=0$) is straightforward since the temporal WL test initializes colors with node features. We now focus on the inductive step. Suppose the proposition holds for $\ell - 1$. Note that our update function:
\begin{align*}
    h_v^{(\ell)}(t) & = \textsc{Update}^{(\ell)}\left(h_v^{(\ell-1)}(t), \textsc{Agg}^{(\ell)}(\{\!\!\{ (h_u^{(\ell-1)}(t), t-t^\prime, e) \mid (u, e, t^\prime) \in \mathcal{N}(v, t)\}\!\!\})\right)
\end{align*}
can be rewritten using $\varphi$ as a function of node colors:
\begin{align*}
    h_v^{(\ell)}(t) & = \textsc{Update}^{(\ell)}\left(\varphi(c^{\ell-1}(v)), \textsc{Agg}^{(\ell)}(\{\!\!\{ (\varphi(c^{\ell-1}(u)), t-t^\prime, e) \mid (u, e, t^\prime) \in \mathcal{N}(v, t)\}\!\!\})\right) .
\end{align*}

Note that the composition of injective functions is also injective. In addition, time-shifting operations are also injective. Thus we can construct an injection $\psi$ such that:
\begin{align*}
        h_v^{(\ell)}(t) &= \psi\left( 
c^{\ell-1}(v), \{\!\!\{ (c^{\ell-1}(u), t^\prime, e) \mid (u, e, t^\prime) \in \mathcal{N}(v, t)\}\!\!\})
        \right) \\
         &= \psi\left( 
c^{\ell-1}(v), \{\!\!\{ (c^{\ell-1}(u), t^\prime, e_{uv}(t^\prime)) \mid (v, u, t^\prime) \in \mathcal{G}(t)\}\!\!\})
        \right)
\end{align*}
since there exists an element $(u, e, t^\prime) \in \mathcal{N}(v, t)$ if and only if there is an event $(u, v, t^\prime) \in \mathcal{G}(t)$ with feature $e_{uv}(t^\prime)=e$.

Then, we can write:
\begin{align*}
    h_v^{(\ell)}(t) &= \psi \circ \textsc{Hash}^{-1} \circ \textsc{Hash}\left( 
c^{\ell-1}(v), \{\!\!\{ (c^{\ell-1}(u), t^\prime, e_{uv}(t^\prime)) \mid (u, v, t^\prime) \in \mathcal{G}(t)\}\!\!\})
        \right)\\
        &=  \psi \circ \textsc{Hash}^{-1}( c^{(\ell)}(v))
\end{align*}

Note $\varphi = \psi \circ \textsc{Hash}^{-1}$ is injective since it is a composition of two injective functions. We then conclude that if the temporal WL test outputs different multisets of colors, then a suitable \mptgn outputs different multisets of embeddings.
\end{proof}

\subsection{Proof of \autoref{cor:properties}: \mptgns and CAWs fail to decide some graph properties}

\textbf{\textcolor{lb}{Statement 1:}} \mptgns fail to decide some graph properties.
\begin{figure}[th]
    \centering
    \includegraphics[width=0.6\textwidth]{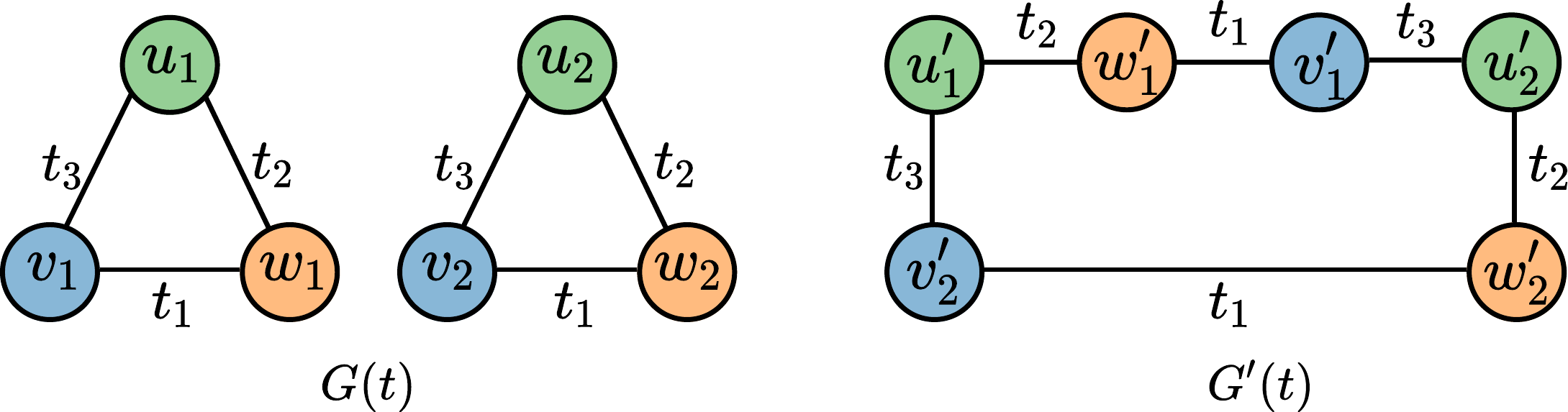}
    \caption{\small Examples of temporal graphs for which \mptgns cannot distinguish the diameter, girth, and number of cycles. For any node in $\mathcal{G}(t)$ (e.g., $u_1$), there is a corresponding one in $\mathcal{G}^\prime(t)$ ($u_1^\prime$) whose TCTs are isomorphic.}
    \label{fig:properties}
\end{figure}

\begin{proof}
Adapting a construction by \citet{Garg2020}, we provide in \autoref{fig:properties} an example that demonstrates \autoref{cor:properties}. Colors denote node features, and all edge features are identical. The temporal graphs $\mathcal{G}(t)$ and $G^\prime(t)$ are non-isomorphic and differ in properties such as diameter ($\infty$ for $\mathcal{G}(t)$ and $3$ for $\mathcal{G}^\prime(t)$), girth ($3$ for $\mathcal{G}(t)$ and $6$ for $\mathcal{G}^\prime(t)$), and number of cycles ($2$ for $\mathcal{G}(t)$ and $1$ for $\mathcal{G}^\prime(t)$). In spite of that, for $t>t_3$, the set of embeddings of nodes in $\mathcal{G}(t)$ is the same as that of nodes in $\mathcal{G}^\prime(t)$ and, therefore, \mptgns cannot decide these properties. In particular, by constructing the TCTs of all nodes at time $t>t_3$, we observe that the TCTs of the pairs $(u_1, u^\prime_1)$, $(u_2, u^\prime_2)$, $(v_1, v^\prime_1)$, $(v_2, v^\prime_2)$, $(w_1, w^\prime_1)$, $(w_2, w^\prime_2)$ are isomorphic and, therefore, they can not be distinguished (\autoref{lemma:1}). 
\end{proof}


\textbf{\textcolor{lb}{Statement 2:}} CAWs fail to decide some graph properties.

Since CAW does not provide a recipe to obtain graph-level embeddings, we first define such a procedure. Let $\mathcal{G}(t)$ be a temporal graph given as a set of events. We sequentially compute event embeddings $h_\gamma$ for each event $\gamma=(u,v,t^\prime) \in \mathcal{G}(t)$ respecting the temporal order (two or more events at the same timestamp are computed in parallel). We then apply a readout layer to the set of event embeddings to obtain a graph-level representation. We provide a proof assuming this procedure.

\begin{figure}[th]
    \centering
    \includegraphics[width=0.6\textwidth]{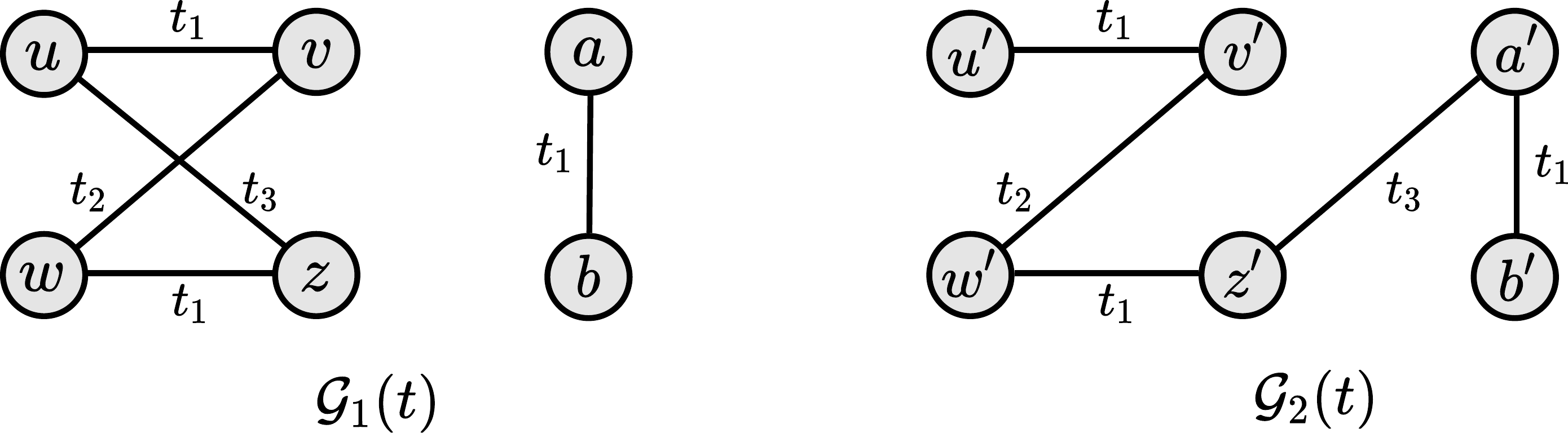}
    \caption{Examples of temporal graphs with different static properties, such as diameter, girth, and number of cycles. CAWs fail to distinguish $\mathcal{G}_1(t)$ and $\mathcal{G}_2(t)$.}
    \label{fig:caw-and-properties}
\end{figure}

\begin{proof}
We can adapt our construction in \autoref{fig:limits-tgns} [rightmost] to extend \autoref{cor:properties} to CAW. The idea consists of creating two temporal graphs with different diameters, girths, and numbers of cycles that comprise events that CAW cannot separate --- \autoref{fig:caw-and-properties} provides one such construction. In particular, CAW obtains identical embeddings for $(u, z, t_3)$ and $(a^\prime, z^\prime, t_3)$ (as shown in \autoref{prop:tgns_vs_caw}). The remaining events are the same up to node re-labelling and thus also lead to identical embeddings. Therefore, CAW cannot distinguish $\mathcal{G}_1(t)$ and $\mathcal{G}_2(t)$ although they  clearly differ in diameter, girth, and number of cycles.
\end{proof}

\subsection{Proof of \autoref{lemma:positional-features}} \label{ap:proof-positional-features}
We now show that the $k$-th component of the relative positional features $r_{u \rightarrow v}^{(t)}$ corresponds to the number of occurrences of $u$ at the $k$-th layer of the {\tcTCT} of $v$, and this is valid for all pairs of nodes $u$ and $v$ of the dynamic graph. We proceed with a proof by induction.

[\emph{Base case}] Let us consider $t=0$, i.e., no events have occurred. By definition, $r_{u \rightarrow v}^{(0)}$ is the zero vector if $u \neq v$, indicating that node $u$ does not belong to the TCT of $v$. If $u=v$, then $r_{u \rightarrow u}^{(0)}=[1, 0, \dots, 0]$ corresponds to count $1$ for the root of the TCT of $u$. Thus, for $t=0$, the proposition holds.

\begin{wrapfigure}[18]{r}{0.45\textwidth}
    \centering
    \includegraphics[width=0.4\textwidth]{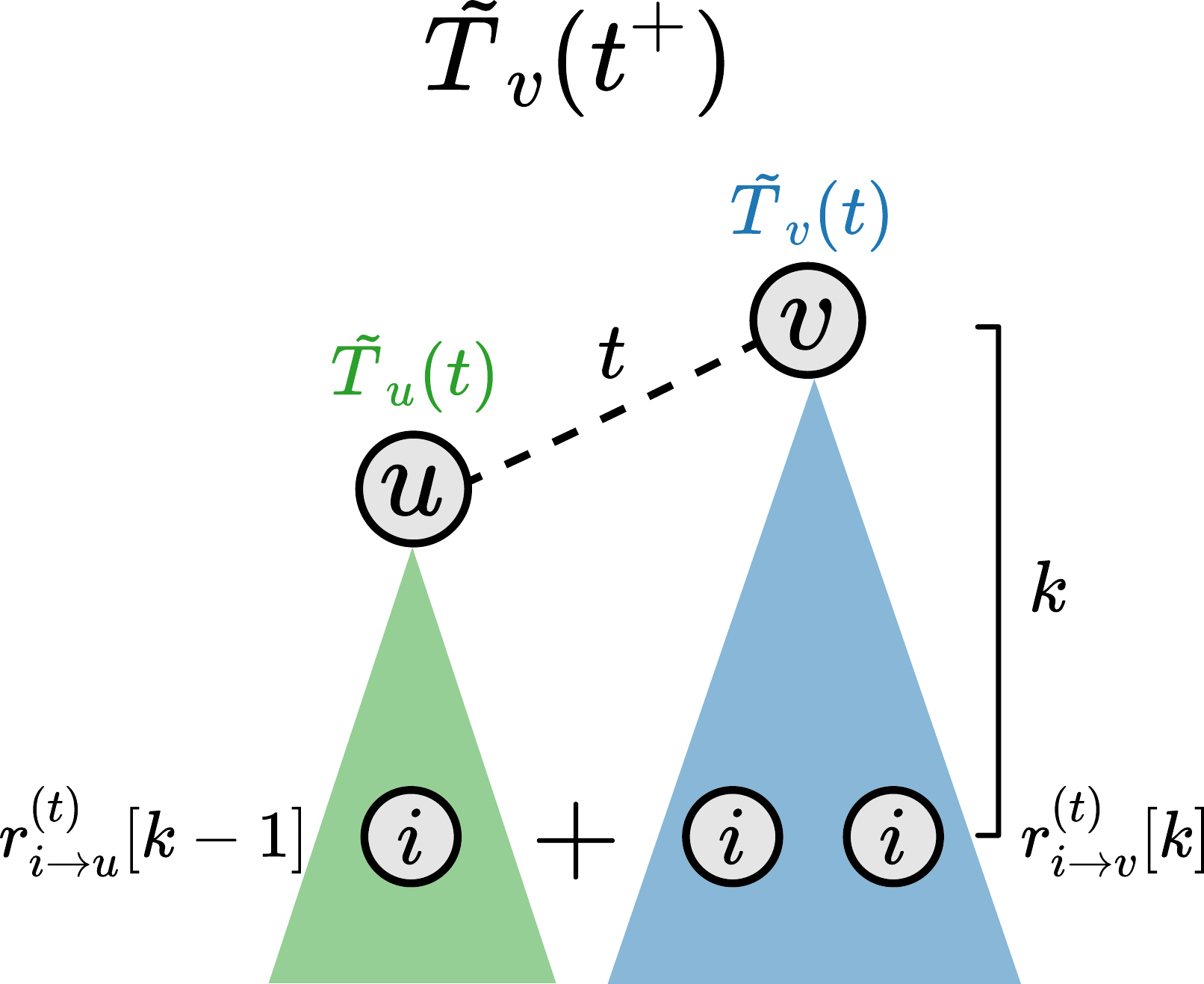}
    \caption{Illustration of how the {\tcTCT} of $v$ changes after an event between $u$ and $v$ at time $t$. This allows us to see how to update the positional features of any node $i$ of the dynamic graph that belongs to $\tilde{T}_u{(t)}$ relative to $v$.}
    \label{fig:positional-features}
\end{wrapfigure}
[\emph{Induction step}] Assume that the proposition holds for all nodes and any time instant up to $t$. We will show that after the event $\gamma=(u, v, t)$ at time $t$, the proposition remains true. 

Note that the event $\gamma$ only impacts the {\tcTCT}s of $u$ and $v$. The reason is that the monotone TCTs of all other nodes have timestamps lower than $t$, which prevents the event $\gamma$ from belonging to any path (with decreasing timestamps) from the root.

Without loss of generality, let us now consider the impact of $\gamma$ on the {\tcTCT} of $v$. \autoref{fig:positional-features} shows how the TCT of $v$ changes after $\gamma$, i.e., how it goes from $\tilde{T}_v(t)$ to $\tilde{T}_v(t^+)$. In particular, the process attaches the TCT of $u$ to the root node $v$. Under this change, we need to update the counts of all nodes $i$ in $\tilde{T}_u{(t)}$ regarding how many times it appears in $\tilde{T}_v{(t^+)}$. We do so by adding the counts in $\tilde{T}_u{(t)}$ (i.e., $r^{(t)}_{i \rightarrow u}$) to $\tilde{T}_v{(t)}$ (i.e., $r^{(t)}_{i \rightarrow v}$), accounting for the $1$-layer mismatch, since $\tilde{T}_u{(t)}$ is attached to the first layer. This can be easily achieved with the shift matrix $P=\begin{bmatrix}
0 & 0 \\
I_{d-1} & 0
\end{bmatrix}$ applied to the counts of any node $i$ in $\tilde{T}_u{(t)}$, i.e.,
\begin{align*}
{r}_{i \rightarrow v}^{(t^+)} &= P~ {r}_{i \rightarrow u}^{(t)} + {r}_{i \rightarrow v}^{(t)} \quad \quad \forall i \in \mathcal{V}_u^{(t)},
\end{align*}
where $\mathcal{V}_u^{(t)}$ comprises the nodes of the original graph that belong to $\tilde{T}_u^{(t)}$.

Similarly, the event $\gamma$ also affects the counts of nodes in the monotone TCT of $v$ w.r.t. the monotone TCT of $u$. To account for that change, we follow the same procedure and update ${r}_{j \rightarrow u}^{(t^+)} = P ~ {r}_{j \rightarrow v}^{(t)} + {r}_{j \rightarrow u}^{(t)}, \forall j \in \mathcal{V}_v^{(t)}$.

\paragraph{Handling multiple events at the same time.}
We now consider the setting where a given node $v$ interacts with multiple nodes $u_1, u_2, \dots, u_J$ at time $t$. We can extend the computation of positional features to this setting in a straightforward manner by noting that each event leads to an independent branch in the TCT of $v$. Therefore, the update of the positional features with respect to $v$ is given by
\begin{align*}
{r}_{i \rightarrow v}^{(t^+)} &= P \sum_{j=1}^J {r}_{i \rightarrow u_j}^{(t)} + {r}_{i \rightarrow v}^{(t)} \quad \quad \forall i \in \bigcup_{j=1}^J \mathcal{V}_{u_j}^{(t)} \\
 \mathcal{V}_{v}^{(t^+)} & = \mathcal{V}_{v}^{(t)} \bigcup_{j=1}^J \mathcal{V}_{u_j}^{(t)}. 
\end{align*}

We note that the updates of the positional features of $u_1, \dots, u_J$ remain untouched if they do not interact with other nodes at time $t$.

\subsection{Proof of \autoref{prop:injective-temporal-aggregation}: Injective function on temporal neighborhood}
\begin{proof} 
To capture the intuition behind the proof, first consider a multiset $M$ such that $|M|<4$. We can assign a unique number $\psi(m) \in \{1, 2, 3, 4\}$ to any distinct element $m \in M$. Also, the function $h(m)=10^{-\psi(m)}$ denotes the decimal expansion of $\psi(m)$ and corresponds to reserving one decimal place for each unique element $m \in M$. Since there are less than $10$ elements in the multiset, note that $\sum_m h(m)$ is unique for any multiset $M$.

To prove the proposition, we also leverage the well-known fact that the Cartesian product of two countable sets is countable --- the Cantor's (bijective) pairing function $z: \mathbb{N} \times \mathbb{N} \rightarrow \mathbb{N}$, with $z(n_1,n_2)=\frac{(n_1+n_2)(n_1+n_2+1)}{2}+n_2$, provides a proof for that.

Here, we consider multisets $M=\{\!\!\{(x_i, e_i, t_i)\}\!\!\}$ whose tuples take values on the Cartesian product of the countable sets $\mathcal{X}$, $\mathcal{E}$, and $\mathcal{T}$ --- the latter is also assumed to be bounded.
In addition, we assume the lengths of all multisets are bounded by $N$, i.e., $|M| < N$ for all $M$. Since $\mathcal{X}$ and $\mathcal{E}$ are countable, there exists an enumeration function $\psi: \mathcal{X} \times \mathcal{E} \rightarrow \mathbb{N}$ for all $M$. Without loss of generality, we assume $\mathcal{T}=\{1, 2, t_{\max} \}$. We want to show that exists a function of the form $\sum_i 10^{- k \psi(x_i, e_i)} \alpha^{-\beta t_i}$ that is unique on any multiset $M$.

Our idea is to separate a range of $k$ decimal slots for each unique element ($x_i, e_i, \cdot$) in the multiset. Each such a range has to accommodate at least $t_{\max}$ decimal slots (one for each value of $t_i$). Finally, we need to make sure we can add up to $N$ values at each decimal slot.

Formally, we map each tuple $(x_i, e_i, \cdot)$ to one of $k$ decimal slots starting from $10^{-k\psi(x_i, e_i)}$. In particular, for each element $(x_i, e_i, t_i=j)$ we add one unit at the  $j$-th decimal slot after $10^{-k\psi(x_i, e_i)}$. 
Also, to ensure the counts for $(x_i, e_i, j)$ and $(x_i, e_i, l\neq j)$ do not overlap, we set $\beta = \lceil\log_{10} N\rceil$ since no tuple can repeat more than $N$ times. We use $\alpha=10$ as we shift decimals. Finally, to guarantee that each range encompasses $t_{\max}$ slots of $\beta$ decimals, we set $k= \beta (t_{\max} + 1)$. 
Therefore, the function 
\begin{align*}
\sum_i 10^{- k \psi(x_i, e_i)} \alpha^{-\beta t_i}
\end{align*}
is unique on any multiset $M$. We note that, without loss of generality, one could choose a different basis (other than 10).
\end{proof}



 

\subsection{Proof of \autoref{prop:expressiveness_sinet}: Expressiveness of PINT: link prediction}

\begin{proof}
We now show that \initials (with relative positional features) is strictly more powerful than \mptgn and CAW in distinguishing edges of temporal graphs. Leveraging \autoref{prop:tgns_vs_caw}, it suffices to show that \initials is at least as powerful as both CAW and \mptgn.

\textbf{\textcolor{lb}{Statement 1:}} \initials is at least as powerful as \mptgns.

Since \initials is a generalization of \mptgns with injective aggregation/update layers, it derives that it is at least as powerful as \mptgns. We can set the model's parameters associated with positional features to zero and obtain an equivalent \mptgn.

\textbf{\textcolor{lb}{Statement 2:}} \initials is at least as powerful as CAW.

We wish to show that for any pair of events that \initials cannot distinguish, CAW also cannot distinguish it. Let us consider the events $(u, v, t)$ and $(u^\prime, v^\prime, t)$ of a temporal graph. Formally, we want to prove that if $\multiset{T_u(t), T_v(t)} = \multiset{T_{u^\prime}(t), T_{v^\prime}(t)}$ (i.e., the multisets contain TCTs that are pairwise isomorphic), then $\multiset{\textsc{Enc}(W; S_u, S_v)}_{W \in \{S_u \cup S_v\}} = \multiset{\textsc{Enc}(W^\prime; S_{u^\prime}, S_{v^\prime})}_{W^\prime \in \{S_{u^\prime} \cup S_{v^\prime}\}}$, where $\textsc{Enc}$ denotes the walk-encoding function of CAW. Importantly, for the sake of this proof, we assume that all TCTs here are augmented with positional features, characterizing edge embeddings obtained from \initials.

Without loss of generality, we can assume that $T_{u}(t) \cong T_{u^\prime}(t)$ and $T_{v}(t) \cong T_{v^\prime}(t)$. By \autoref{lemma:isomorphism-time-contrained-TCT}, we know that the corresponding {\tcTCT}s are also isomorphic: $\tilde{T}_u(t) \cong \tilde{T}_{u^\prime}(t)$, and $\tilde{T}_{v}(t) \cong \tilde{T}_{v^\prime}(t)$ with associated bijections $f_1: V(\tilde{T}_u(t)) \rightarrow V(\tilde{T}_{u^\prime}(t))$ and $f_2: V(\tilde{T}_v(t)) \rightarrow V(\tilde{T}_{v^\prime}(t))$.

We can construct a tree $T_{uv}$ by attaching $\tilde{T}_{u}(t)$ and $\tilde{T}_{v}(t)$ to a (virtual) root node $uv$ --- without loss of generality, $u$ and $v$ are the left-hand and right-hand child of $uv$, respectively. We can follow the same procedure and create the tree $T_{u^\prime v^\prime}$ by attaching the TCTs $\tilde{T}_{u^\prime}(t)$ and $\tilde{T}_{v^\prime}(t)$ to a root node $u^\prime v^\prime$. Since the left-hand and right-hand subtrees of $T_{uv}$ and $T_{u^\prime v^\prime}$ are isomorphic, then $T_{uv}$ and $T_{u^\prime v^\prime}$ are also isomorphic.
Let $f: V(T_{uv}) \rightarrow V(T_{u^\prime v^\prime})$ denote the bijection associated with the augmented trees. We also assume $f$ is constructed by preserving the bijections $f_1$ and $f_2$ defined between the original {\tcTCT}s: this ensures that $f$ does not map any node in $V(\tilde{T}_u(t))$ to a node in $V(\tilde{T}_{v^\prime}(t))$, for instance.
We have that
\begin{align*}
    [r^{(t)}_{\sharp i \rightarrow u} \| r^{(t)}_{\sharp i \rightarrow v}]  = [r^{(t)}_{\sharp f(i) \rightarrow u^\prime} \| r^{(t)}_{\sharp f(i) \rightarrow v^\prime}] \quad \forall i \in V(T_{uv}) \setminus \{uv\}
\end{align*}

Note that we use the function $\sharp$ (that maps nodes in the TCT to nodes in the dynamic graph) here because the positional feature vectors are defined for nodes in the dynamic graph.

To guarantee that two encoded walks are identical $\textsc{Enc}(W; S_u, S_v)=\textsc{Enc}(W^\prime; S_{u^\prime}, S_{v^\prime})$, it suffices to show that the anonymized walks are equal. Thus, we turn our problem into showing that for any walk $W=(w_0, t_0, w_1, t_1, \dots)$ in $S_u \cup S_v$, there exists a corresponding one $W^\prime=(w^\prime_0, t_0, w^\prime_1, t_1, \dots)$ in $S_{u^\prime} \cup S_{v^\prime}$ such that $I_{\text{CAW}}(w_i; S_u, S_v) = I_{\text{CAW}}(w^\prime_i; S_{u^\prime}, S_{v^\prime})$ for all $i$. Recall that $I_{\text{CAW}}(w_i; S_u, S_v)=\{g(w_i; S_u), g(w_i; S_v)\}$, where $g(w_i; S_u)$ is a vector whose $k$-component stores how many times $w_i$ appears in a walk from $S_u$ at position $k$. 

A key observation is that there is an equivalence between deanonymized root-leaf paths in $T_{uv}$ and walks in $S_u \cup S_v$ (disregarding the virtual root node). By deanonymized, we mean paths where node identities (in the temporal graph) are revealed by applying the function $\sharp$. Using this equivalence, it suffices to show that
\begin{align*}
g(\sharp i; S_u) = g(\sharp f(i); S_{u^\prime}) \text{ and } g(\sharp i; S_v) = g(\sharp f(i); S_{v^\prime}) \quad \forall i \in V(T_{uv}) \setminus \{uv\}
\end{align*}

Suppose there is an $i \in V(T_{uv}) \setminus \{uv\}$ such that $g(\sharp i; S_u) \neq g(\sharp f(i) ; S_{u^\prime})$. Without loss of generality, suppose this holds for the $\ell$-th entry of the vectors.

We know there are exactly $r^{(t)}_{a \rightarrow u}[\ell]$ nodes at the $\ell$-th level of $\tilde{T}_u(t)$ that are associated with $a = \sharp i \in V(\mathcal{G}(t))$. We denote by $ \Psi$ the set comprising such nodes. It also follows that computing $g(\sharp i; S_u)[\ell]$ is the same as summing up the amount of leaves of each subtree of $\tilde{T}_u(t)$ rooted at $\psi \in \Psi$, which we denote as $\myleaf(\psi; \tilde{T}_u(t))$, i.e.,
\begin{equation*}
   g(\sharp i; S_u)[\ell] = \sum_{\psi \in \Psi} 
   \myleaf(\psi; \tilde{T}_u(t)).
\end{equation*}

Since we assume $g(\sharp i; S_u)[\ell] \neq g(\sharp f(i) ; S_{u^\prime})[\ell]$, then it holds that
\begin{equation}\label{eq:leaves}
g(\sharp i; S_u)[\ell] \neq g(\sharp f(i); S_{u^\prime})[\ell]
\Rightarrow \sum_{\psi \in \Psi} \myleaf(\psi; \tilde{T}_u(t)) \neq \sum_{\psi \in \Psi} \myleaf(f(\psi); \tilde{T}_{u^\prime}(t))
\end{equation}

Note that the {subtree of $\tilde{T}_u$ rooted at $\psi$} should be isomorphic to the subtree of $\tilde{T}_{u^\prime}$ rooted at $f(\psi)$, and therefore have the same number of leaves. However, the RHS of \autoref{eq:leaves} above implies there is a $\psi \in \Psi$ for which $\myleaf(\psi; \tilde{T}_u) \neq  \myleaf(f(\psi); \tilde{T}_{u^\prime})$, reaching a contradiction. The same argument can be applied to $v$ and $v^\prime$ to prove that $g(\sharp i; S_v) = g(\sharp f(i); S_{v^\prime})$.
\end{proof}

\section{Additional related works} \label{ap:related}

\paragraph{Structural features for static GNNs.} Using structural features to enhance the power of GNNs is an active research topic.
\citetsup{Bouritsas2020} improved GNN expressivity by incorporating counts of local structures in the message-passing procedure, e.g, the number triangles a node appears on.
These counts depend on identifying subgraph isomorphisms and, naturally, can become intractable depending on the chosen substructure.
\citetsup{Li2020} proposed increasing the power of GNNs using distance encodings, i.e., augmenting original node features with distance-based ones. 
In particular, they compute the distance between a node set whose representation is to be learned and each node in the graph.
To alleviate the cost of distance encoding, an alternative is to learn absolute position encoding schemes \citepsup{Kreuzer2021,Wang2022,Srinivasan2020}, that try to summarize the role each node plays in the overall graph topology. 
We note that another class of methods uses random features to boost the power of GNNs \citepsup{Sato2021RandomFS,Abboud2021}. 
However, these models are referred to be hard to converge and obtain noisy predictions \citepsup{Wang2022}.

The most trivial difference between these approaches and \initials is that our relative positional features account for temporal information.
On a deeper level, our features summarize the role each node plays in each other's \tcTCT  instead of measuring, e.g., pair-wise distances in the original graph or counting substructures.
Also, our scheme leverages the temporal aspect to achieve computational tractability, updating features incrementally as events unroll.
Finally, while some works proposing structural features for static GNNs present marginal gains \citesup{Wang2022},
\initials exhibits significant performance gains in real-world temporal link prediction tasks.

\paragraph{Other models for temporal graphs.} Representation learning for dynamic graphs is a broad and diverse field. 
In fact, strategies to cope with the challenge of modeling dynamic graphs can come in many flavors, including simple aggregation schemes \citep{Kleinberg2007}, walk-aggregating methods \citepsup{Mahdavi2018}, and combinations of sequence models with GNNs \citepsup{Seo2018,Pareja2020}. 
For instance, \citetsup{Seo2018} used a spectral graph convolutional network \citepsup{Defferrard16} to encode graph snapshots followed by a graph-level LSTM \citep{lstm}. 
\citetsup{Manessi2020} followed a similar approach but employed a node-level LSTM, with parameters shared across the nodes.
\citetsup{Sankar2020} proposed a fully attentive model based on graph attention networks \citep{gat}.
\citetsup{Pareja2020} applied a recurrent neural net to dynamically update the parameters of a GCN.
\citetsup{Gao2021} compared the expressive power of two classes of models for discrete dynamic graphs: time-and-graph and time-then-graph. The former represents the standard approach of interleaving GNNs and sequence (e.g., RNN) models. In the latter class, the models first capture node and edge dynamics using RNNs, and are then feed into graph neural networks. The authors showed that time-then-graph has expressive advantage over time-and-graph approaches under mild assumptions. 
For an in-depth review of representation learning for dynamic graphs, we refer to the survey works by \citet{Kazemi2020} and \citetsup{Skarding2021FoundationsAM}. 

While most of the early works focused on discrete-time dynamic graphs, we have recently witnessed a rise in interest in models for event-based temporal graphs (i.e., CTDGs).
The reason is that models for DTDGs may fail to leverage fine-grained temporal and structural information that can be crucial in many applications. In addition, it is hard to specify meaningful time intervals for different tasks. Thus, modern methods for temporal graphs explicitly incorporate timestamp information into sequence/graph models, achieving significant performance gains over approaches for DTDGs \citep{caw}. 
\autoref{ap:egn} provides a more detailed presentation of CAW, TGN-Att, and TGAT, which are among the best performing models for link prediction on temporal graphs.
Besides these methods, JODIE~\citep{jodie} applies two RNNs (for the source and target nodes of an event) with a time-dependent embedding projection to learn node representations of item-user interaction networks.
\citet{DyRep} employed RNNs with a temporally attentive module to update node representations. 
APAN \citepsup{APAN2021} consists of a memory-based TGN that uses attention mechanism to update memory states using multi-hop temporal neighborhood information. 
\citetsup{tgn-caw} proposed incorporating edge embeddings obtained from CAW into \mptgns' memory and message-passing computations.

\newpage
\section{Datasets and implementation details}\label{ap:implementation} 

\subsection{Datasets}

In our empirical evaluation, we have considered six datasets for dynamic link prediction: Reddit\footnote{\url{http://snap.stanford.edu/jodie/reddit.csv}}, Wikipedia\footnote{\url{http://snap.stanford.edu/jodie/wikipedia.csv}}, UCI\footnote{\url{http://konect.cc/networks/opsahl-ucforum/}}, LastFM\footnote{\url{http://snap.stanford.edu/jodie/lastfm.csv}}, Enron\footnote{\url{https://www.cs.cmu.edu/~./enron/}}, and Twitter.
Reddit is a network of posts made by users on subreddits, considering the 1,000 most active subreddits and the 10,000 most active users. 
Wikipedia comprises edits made on the 1,000 most edited Wikipedia pages by editors with at least 5 edits. 
Both Reddit and Wikipedia networks include links collected over one month, and text is used as edge features, providing informative context.
The LastFM dataset is a network of interactions between user and the songs they listened to.
UCI comprises students' posts to a forum at the University of California Irvine.
Enron contains a collection of email events between employees of the Enron Corporation, before its bankruptcy.
The Twitter dataset is a non-bipartite net where nodes are users and interactions are retweets. 
Since Twitter is not publicly available, we build our own version by following the guidelines by \citet{tgn}. 
We use the data available from the 2021 Twitter RecSys Challenge and select 10,000 nodes and their associated interactions based on node participation: number of interactions the node participates in. 
We also apply multilingual BERT to obtain text representations of retweets (edge features).

\autoref{tab:summary_link} reports statistics of the datasets such as number of temporal nodes and links, and the dimensionality of the edge features. 
We note that UCI, Enron, and LastFM represent non-attributed networks and therefore do not contain feature vectors associated with the events. 
Also, the node features for all datasets are vectors of zeros \citep{tgat}.

\begin{table}[h]
\centering
\caption{Summary statistics of the datasets.}
\begin{tabular}{l cccc}
\toprule
\textbf{Dataset} & \textbf{\#Nodes} & \textbf{\#Events} &   \textbf{\#Edge feat.}  & \textbf{Bipartite?}\\ \midrule
Reddit           &  10,984 (10,000 / 984)          &   672,447        &        172     & Yes           \\
Wikipedia        &  9,227 (8,227 / 1,000)           &   157,474        &        172         & Yes       \\
Twitter & 8,925 & 406,564 & 768  & No \\
UCI              &  1,899           &    59,835        & -  & No \\
Enron & 184 & 125,235 & - & No \\
LastFM & 1,980 (980 / 1,000) & 1,293,103 & - & Yes \\
\bottomrule
\end{tabular}
\label{tab:summary_link}
\end{table}

\subsection{Implementation details}

We train all models in link prediction tasks in a self-supervised approach. 
During training, we generate negative samples: for each actual event $(z, w, t)$ (class 1), we create a fake one $(z, w^\prime, t)$ (class 0) where $w^\prime$ is uniformly sampled from the set of nodes, and both events have the same edge feature vector.

To ensure a fair comparison, we mainly rely on the original repositories and guidelines. For instance, regarding the training of \mptgns (including \initials), we mostly follow the setup and choices in the implementation available in \citep{tgn}. In particular, we apply the Adam optimizer with learning rate $10^{-4}$ during $50$ epochs with early stopping if there is no improvement in validation AP for $5$ epochs. In addition, we use batch size 200 for all methods. 
We report statistics (mean and standard deviation) of the performance metric (AP) over ten runs.

\paragraph{\mptgns.} For TGN-Att, we follow \citet{tgn} and sample either ten or twenty temporal neighbors with memory dimensionality equal to 32 (Enron), 100 (UCI, Twitter), or 172 (Reddit, Wikipedia, LastFM), node embedding dimension equal to 100, and two attention heads. 
We use a memory unit implemented as a GRU, and update the state of each node based on only its most recent message.
For TGAT, we use twenty temporal neighbors and two layers. 
%

\paragraph{CAW.} We conduct model selection using grid search over: i) time decay $\alpha \in \{0.01, 0.1, 0.25, 0.5, 1.0, 2.0, 4.0, 10.0, 100.0\} \times 10^{-6}$, ii) number of walks $M \in \{1, 2, 3, 4, 5\}$;  and iii) walk length $L \in \{32, 64, 128\}$. The best combination of hyperparameters is shown in \autoref{tab:caw_hyperparam}. 
The remaining training choices follows the default values from the original implementation.
Importantly, we note that TGN-Att's original evaluation setup is different from CAW's. Thus, we adapted CAW's original repo to reflect these differences and ensure a valid comparison.

\begin{table}[h]
\centering
\caption{Optimal hyperparameters for CAW.}
\begin{tabular}{l ccc cc cc}
\toprule
\textbf{Dataset} &\textbf{Time decay} $\mathbf{\alpha}$ & \textbf{\#Walks} &   \textbf{Walk length} \\ \midrule
Reddit & $10^{-8}$ & $32$ & $3$ \\
Wikipedia & $4\times10^{-6}$ & $64$ & $4$ \\
Twitter & $10^{-6}$ & $64$ & $3$ \\
UCI & $10^{-5}$ & $64$ & $2$ \\
Enron & $10^{-6}$ & $64$ & $5$ \\
LastFM  & $10^{-6}$ & $64$ & $2$\\
\bottomrule
\end{tabular}
\label{tab:caw_hyperparam}
\end{table}

\paragraph{\initials.} We use $\alpha=2$ (in the exponential aggregation function), and experiment with learned and fixed $\beta$. We apply a relu function to avoid negative values of $\beta$, which could lead to unstable training. We do grid search as the follow: when learning beta, we consider initial values for $\beta \in \{0.1, 0.5\}$; for the fixed case (\texttt{requires\_grad=False}), we evaluate $\beta \in \{10^{-3} \times |\mathcal{N}|, 10^{-4}\times| \mathcal{N}|, 10^{-5}\times| \mathcal{N}|\}$ --- $|\mathcal{N}|$ denotes the number of temporal neighbors --- and always apply memory as in the original implementation of TGN-Att. We consider number message passing layers $\ell$ in $\{ 1, 2\}$. Also, we apply neighborhood sampling with the number of neighbors in $\{10, 20\}$, and update the state of a node based on its most recent message.  We then carry out model selection based on AP values obtained during validation. Overall, the models with fixed $\beta$ led to better results. \autoref{tab:pint_hyperparam} reports the optimal hyperparameters for \initials found via automatic model selection. 

In all experiments, we use relative positional features with $d=4$ dimensions. For computational efficiency, we update the relative positional features only after processing a batch, factoring in all events from that batch. Note that this prevents information linkage as these positional features take effect after prediction. In addition, since temporal events repeat (in the same order) at each epoch, we also speed up \initials's training procedure by precomputing and saving the positional features for each batch. To save up space, we store the positional features as sparse matrices. 

\begin{table}[h]
\centering
\caption{Optimal hyperparameters for \initials.}
\begin{tabular}{l ccc}
\toprule
\textbf{Dataset} & $\beta/|\mathcal{N}|$ & \textbf{\#Neighbors} ($|\mathcal{N}|$) &   \textbf{\#Layers} \\ \midrule
Reddit & $10^{-5}$ & $10$ & $2$ \\
Wikipedia & $10^{-4}$ & $10$ & $2$ \\
Twitter & $10^{-5}$ & $20$ & $2$ \\
UCI & $10^{-5}$ & $10$ & $2$ \\
Enron & $10^{-5}$ & $20$ & $2$ \\
LastFM & $10^{-4}$ & $10$ & $2$ \\
\bottomrule
\end{tabular}
\label{tab:pint_hyperparam}
\end{table}

\paragraph{Hardware.} For all experiments, we use Tesla V100 GPU cards and consider a memory budget of 32GB of RAM. 

\section{Deletion and node-level events}\label{ap:deletion}

\citet{tgn} propose handling edge deletions by simply updating memory states of the edge's endpoints, as if we were dealing with a usual edge addition. However, it is not discussed whether the edge in question should be excluded from the event list or if we should just add a novel event with edge features that characterize deletion. If we choose the former, we may be unable to recover the memory state of a node from its \tcTCT and the original node features. Removing an edge from the event list also affects the computation of node embeddings. Therefore, we advise practitioners to do the latter when using \initials. 
It is worth mentioning that the vast majority of models for temporal interaction prediction do not consider the possibility of deletion events.

Regarding node-level events, \initials can accommodate node addition by simply creating novel memory states. To deal with node feature updates, we can create an edge event with both endpoints on that node, inducing a self-loop in the dynamic graph. Also, we can combine (e.g., concatenate) the temporal features in message-passage operations, similarly to the general formulation of the \mptgn framework \citep{tgn}. Finally, we can deal with the removal of a node $v$ by following our previous (edge deletion) procedure to delete all edges with endpoints in $v$.




\section{Additional experiments}\label{ap:additional-experiments}

\paragraph{Time comparison.}\autoref{fig:time_comparison_sup} compares the time per epoch for \initials and for the prior art (CAW, TGN-Att, and TGAT) in the Enron and LastFM datasets. Following the trend in \autoref{fig:time-comparison},  \autoref{fig:time_comparison_sup} further supports that \initials is generally slower than other \mptgns but, after a few training epochs, is orders of magnitude faster than CAW. In the case of Enron, the time CAW takes to complete an epoch is much higher than the time we need to preprocess \initials's positional features.

\begin{figure}[ht]
    \centering
    \includegraphics[width=0.7\textwidth]{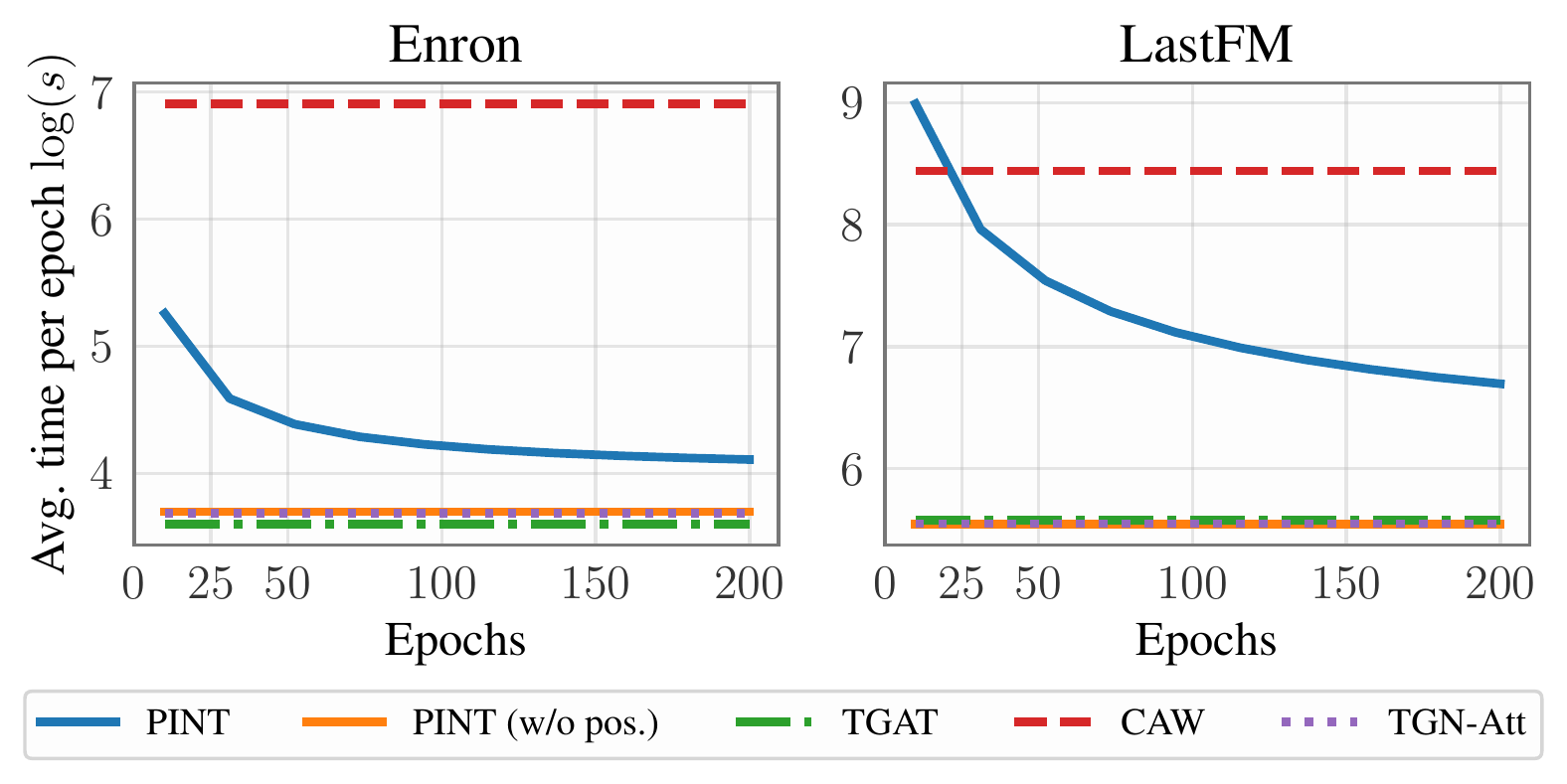}
    \caption{Time comparison: \initials versus TGNs  (in log-scale) on Enron and LastFM.}
    \label{fig:time_comparison_sup}
\end{figure}

\paragraph{Experiments on node classification.} For completeness, we also evaluate PINT on node-level tasks (Wikipedia and Reddit). We follow closely the experimental setup in \citet{tgn} and compare against the baselines therein. \autoref{tab:node} shows that PINT ranks first on Reddit and second on Wikipedia. The values for PINT reflect the outcome of $5$ repetitions.

\begin{table}[h]
\centering
\caption{Results for node classification (AUC).}
\begin{tabular}{l cc}
\toprule
 & \textbf{Wikipedia} &   \textbf{Reddit} \\ \midrule
CTDNE & 75.89 $\pm$ 0.5 & 59.43 $\pm$ 0.6\\
JODIE & 84.84 $\pm$ 1.2 & 61.83  $\pm$ 2.7\\
TGAT & 83.69 $\pm$ 0.7 & 65.56 $\pm$ 0.7\\
DyRep & 84.59 $\pm$ 2.2 & 62.91 $\pm$ 2.4 \\
TGN-Att & 87.81 $\pm$ 0.3 & 67.06 $\pm$ 0.9 \\
PINT & 87.59 $\pm$  0.6 & 67.31 $\pm$ 0.2\\
\bottomrule
\end{tabular}
\label{tab:node}
\end{table}

{
\small
\bibliographystylesup{plainnat} 
\bibliographysup{references}
}

\end{document}